\documentclass[conference]{IEEEtran}
\usepackage{times}
\usepackage[nolist]{acronym} % acronyms by the \ac{label} command
\usepackage{multicol}
\usepackage{amssymb}
\usepackage{bm}
\usepackage{amsmath}
\usepackage{amsthm}
\usepackage{graphicx}
\usepackage{xcolor}
\usepackage{import}
\usepackage{mathtools}
\usepackage{macros}
\usepackage{svg}
\usepackage{soul}
\usepackage{makecell}
\usepackage{multirow}
\usepackage{array}
\usepackage{colortbl}
\usepackage{siunitx}
\usepackage{adjustbox}
\usepackage{adjustbox}
\usepackage{booktabs}
\usepackage{wrapfig}
\usepackage{url}
\usepackage{xfrac}
\usepackage[framemethod=tikz]{mdframed}

\newtheorem{proposition}{Proposition}
\newtheorem{theorem}{Theorem}
\newtheorem{corollary}{Corollary}
\newtheorem{remark}{Remark}

\newtheorem{conjecture}{Conjecture}

\definecolor{lightlightgray}{rgb}{0.92, 0.92, 0.92}

\usepackage[numbers]{natbib}
\usepackage{multicol}
\usepackage[bookmarks=true]{hyperref}

\begin{document}

\title{Reduced-order Control and Geometric Structure\\ of Learned Lagrangian Latent Dynamics}

% You will get a Paper-ID when submitting a pdf file to the conference system
% \author{Author Names Omitted for Anonymous Review. Paper-ID 674}

% \author{\authorblockN{Michael Shell}
% \authorblockA{School of Electrical and\\Computer Engineering\\
% Georgia Institute of Technology\\
% Atlanta, Georgia 30332--0250\\
% Email: mshell@ece.gatech.edu}
% \and
% \authorblockN{Homer Simpson}
% \authorblockA{Twentieth Century Fox\\
% Email: homer@thesimpsons.com}
% \and
% \authorblockN{James Kirk\\ and Montgomery Scott}
% \authorblockA{Starfleet Academy\\
% San Francisco, California 96678-2391\\
% Telephone: (800) 555--1212\\

% avoiding spaces at the end of the author lines is not a problem with
% conference papers because we don't use \thanks or \IEEEmembership

\author{\authorblockN{Katharina Friedl\authorrefmark{1},
No\'{e}mie Jaquier\authorrefmark{1},
Seungyeon Kim\authorrefmark{1}, 
Jens Lundell\authorrefmark{2} and
Danica Kragic\authorrefmark{1}}
\authorblockA{\authorrefmark{1}Division of Robotics, Perception, and Learning, KTH Royal Institute of Technology}
\authorblockA{\authorrefmark{2}University of Turku \\ \texttt{\small \{kfriedl,jaquier,seukim,dani\}@kth.se}, \texttt{\small jens.lundell@utu.fi}}
% \authorblockA{\authorrefmark{3}Starfleet Academy, San Francisco, California 96678-2391\\
% Telephone: (800) 555--1212, Fax: (888) 555--1212}
% \authorblockA{\authorrefmark{4}Tyrell Inc., 123 Replicant Street, Los Angeles, California 90210--4321}
}

\maketitle
\newacro{ae}[AE]{Autoencoder}
\newacro{rom}[ROM]{reduced-order model}
\newacro{fom}[FOM]{full-order model}
\newacro{mor}[MOR]{model order reduction}
\newacro{lnn}[LNN]{Lagrangian neural network}
\newacro{hnn}[HNN]{Hamiltonian neural network}
\newacro{dof}[DoF]{degrees-of-freedom}
\newacro{ivp}[IVP]{initial value problem}
\newacro{dnn}[DNN]{deep neural network}
\newacro{spd}[SPD]{symmetric positive-definite}
\newacro{mlp}[MLP]{multilayer perceptron}
\newacro{fc}[FC]{fully-connected}
\newacro{gyrospd}[GyroSpd$_{\ty{++}}$]{gyrospace hyperplane-based}
\newacro{gyroai}[GyroAI]{gyrocalculus-based}
\newacro{rolnn}[RO-LNN]{reduced-order Lagrangian neural network}
\newacro{rohnn}[RO-HNN]{reduced-order Hamiltonian neural network}
\newacro{pde}[PDE]{partial differential equation}
\newacro{hnko}[HNKO]{Hamiltonian neural Koopman operator}
\newacro{mpc}[MPC]{model-predictive control}
\newacro{rl}[RL]{reinforcement learning}
\newacro{pd}[PD]{proportional derivative}
\newacro{iss}[ISS]{input-to-state stable}
\begin{abstract}
Model-based controllers can offer strong guarantees on stability and convergence by relying on physically accurate dynamic models. However, these are rarely available for high-dimensional mechanical systems such as deformable objects or soft robots.
While neural architectures can learn to approximate complex dynamics, they are either limited to low-dimensional systems or provide only limited formal control guarantees due to a lack of embedded physical structure.
This paper introduces a latent control framework based on learned structure-preserving reduced-order dynamics for high-dimensional Lagrangian systems.
We derive a reduced tracking law for fully actuated systems and adopt a Riemannian perspective on projection-based model-order reduction to study the resulting latent and projected closed-loop dynamics. 
By quantifying the sources of modeling error, we derive interpretable conditions for stability and convergence. 
We extend the proposed controller and analysis to underactuated systems by introducing learned actuation patterns. 
Experimental results on simulated and real-world systems validate our theoretical investigation and the accuracy of our controllers. 

\end{abstract}

\IEEEpeerreviewmaketitle

\section{Introduction}
\label{sec:introduction}
Controlling dynamical systems to precisely track reference trajectories with stability and convergence guarantees is a central challenge in robotics. From rigid manipulators to floating-base platforms, many systems commonly admit a Lagrangian or Hamiltonian formulation with nonlinear dynamics governed by energy conservation laws~\cite{BulloLewis,murray1994}. 
By leveraging this physical structure, controllers can achieve strong performance guarantees. 
Classic examples range from feedback linearization 
to passivity-based and energy-shaping frameworks~\cite{ortega02idpbca,takegaki81} that leverage the system's underlying structure to design stable controllers. 
Related approaches include PD+ and comparable feedforward strategies~\cite{kelly05,padenpanja88}, as well as adaptive laws~\cite{slotineli91}. These control schemes have been broadly extended to task-space control and redundancy resolution~\cite{khatib87,Khatib22:OSC}, underactuated systems~\cite{bloch2002cl,spong94partial}, and contact-rich tasks~\cite{hogan1984impedance, sentis2006wholebody}. 
Their practical success, however, highly depends on the availability and accuracy of the underlying dynamic model. 

The robotics community has recently seen a rising interest in complex dynamical systems, particularly deformable objects and soft manipulators.
While these systems still adhere to the well-studied modeling assumptions of classical mechanics, they are characterized by coupled dynamics evolving on intrinsically high-dimensional configuration spaces~\cite{santina23softrobots,yin2021modeling}. 
Modeling is further complicated by the need for spatial discretization and/or unknown material properties, limiting the availability of suitable models for nonlinear control. 
For these systems, relying on a known dynamic model is rarely feasible. 

To circumvent modeling, some learning-based approaches infer control laws directly from trajectory data~\cite{levine2016, levine2020offline}. However, such black-box controllers overlook the physical nature of the dynamics and cannot be systematically analyzed, leading to overfitting and task-specific strategies.
Alternatively, learned neural dynamics predictors, e.g. coupled with \ac{mpc}~\cite{bruder2019koopmanmpc,zhang2024particle}, 
significantly improve performance and generalization. Yet, standard neural architectures remain difficult to interpret and formally verify.

Physics-inspired dynamic models tackle this issue by incorporating inductive biases arising from physical laws --- e.g., Lagrangian or Hamiltonian mechanics --- leading to improved accuracy and sample efficiency, as well as increased interpretability~\cite{CranmerGreydanus2020LNN,Greydanus2019HNN,Lutter2023DeLaN}. In turn, the learned physical dynamics are compatible with well-studied model-based controllers and their analysis. 
However, as these models scale poorly with increasing configuration space dimensionality~\cite{friedl2025reduced}, existing works are largely constrained to low-dimensional dynamical systems~\cite{duong2021hamiltonian,liu2024lnns,Lutter2023DeLaN, zhong2020symplectic}.  
Note that approaches that learn latent dynamic models from high-dimensional observations, e.g., pixels~\cite{stolzle24oscillatornetworks, Zhong2020Unsupervised}, still operate on low-dimensional configuration spaces, typically of the same dimensionality as the learned latent representation.  
Recently, \acp{rolnn}~\cite{friedl2025reduced} were introduced as a physics-inspired neural architecture to learn the unknown dynamics of high-dimensional Lagrangian systems. \acp{rolnn} build on Riemannian geometry and projection-based nonlinear \ac{mor}~\cite{Buchfink2024MOR} to learn a low-dimensional, structure-preserving surrogate dynamic model. Yet, their implications for model-based control remain unexplored.

Generally, \ac{mor} follows either intrusive paradigms~---~deriving reduced surrogates from known \ac{fom} dynamics~\cite{Buchfink2024MOR,Carlberg2015LagrMOR, Schilders2008MOR}~---~or non-intrusive approaches identifying a low-dimensional representation of the dynamics strictly from data~\cite{friedl2025reduced,SharmaKramer2024:LOpInf}. Regardless of the intrusiveness, preserving the stucture of the \ac{fom} in the \ac{rom} is often advantageous, particularly for physical systems~\cite{Buchfink2024MOR, Carlberg2015LagrMOR}.
While \acp{rom} are well-established for simulation and analysis~\cite{Buchfink2024MOR,Carlberg2015LagrMOR, Otto2023MOR, Schilders2008MOR}, their control-theoretic implications remain insufficiently understood. 
\citet{alora2023ssmmpc} designed a \ac{mpc} scheme on spectral submanifolds, while \citet{lutkus25latent} analyzed the stability and safety of controllers in the latent space of \acp{ae}. However, both frameworks disregard the intrinsic physical structure of the system, and only the latter benefits from the expressivity of neural architectures. A controller based on a Hamiltonian \ac{rom} obtained via intrusive, structure-preserving \ac{mor} led to promising results on regulation tasks~\cite{Lepri2024MOR}, under the restrictive assumptions of a known \ac{fom} and sufficiently small modeling errors. Yet, the closed-loop behavior of \ac{rom}-based controllers remains largely unstudied, especially in non-intrusive scenarios where limited prior knowledge exacerbates modeling uncertainty.

In this paper, we design and analyze \ac{rom}-based tracking controllers for high-dimensional Lagrangian systems with unknown dynamics. The proposed control laws are synthesized from the structure-preserving reduced-order dynamics learned by non-intrusive \acp{rolnn} (Sec.~\ref {sec:method_part1_control}). Taking a Riemannian perspective on projection-based \ac{mor}, we exploit the physical consistency of the learned \ac{rom} to provide a principled analysis of the induced closed-loop dynamics (Sec.~\ref{sec:method_part1_latent_convergence_analysis}-\ref{sec:method_part1_fullorder_plant_convergence_analysis}). For the first time, we characterize the virtual underactuation induced by reduced-order control and explicitly quantify the disturbances arising from projection-induced modeling errors. We establish the conditions for stability and convergence, proving local exponential input-to-state stability for fully-actuated systems, and discuss their extension to indirectly-actuated and underactuated systems (Sec.~\ref{sec:method_part2}). We validate our controllers via experiments on a simulated pendulum and a soft plush puppet manipulated by a real humanoid robot (Sec.~\ref{sec:experiments}).

In summary, we contribute \emph{(1)} a novel \ac{rom}-based tracking control scheme for high-dimensional systems with unknown Lagrangian dynamics, \emph{(2)} its principled characterization with proven locally exponential input-to-state stability, and \emph{(3)} its extension to indirectly- and underactuated systems.

\vspace{-0.05cm}
\section{Preliminaries}
\label{sec:background}
\vspace{-0.05cm}
This section provides theoretical preliminaries on the Riemannian geometry of mechanical control systems, on structure-preserving reduced-order models, and on \acp{rolnn}.

\vspace{-0.05cm}
\subsection{A Primer on Riemannian Geometry}
\vspace{-0.05cm}
A smooth manifold $\mathcalm$ of dimension $n$ can be intuitively conceptualized as a manifold that is locally, but not globally, similar to the Euclidean space $\euclideanspace^n$. Its smooth structure allows the definition of the derivative of curves on the manifold, which are tangent vectors. The set of all tangent vectors at a point $\bm{x}\in\mathcalm$ is a $n$-dimensional vector space called the tangent space $\tangentm{\bm{x}}$. 
The tangent bundle $\tangentbundlem$ is the disjoint union of all tangent spaces and is a $2n$-dimensional smooth manifold.
The cotangent space ${\cotangentm{\bm{x}}=\{\bm{\lambda} | \bm{\lambda}:\tangentm{\bm{x}}\to\euclideanspace \;\text{linear}\}}$ is the dual of $\tangentm{\bm{x}}$. Each cotangent vector $\bm{\lambda}\in\cotangentm{\bm{x}}$ is thus a linear map. \looseness-1
A Riemannian manifold is a smooth manifold $\mathcalm$ endowed with a Riemannian metric $g$, i.e., a smoothly-varying inner product $g_{\bm{x}}\!:\! \tangentm{\bm{x}}\times\tangentm{\bm{x}}\to\euclideanspace$. The Riemannian metric is represented in coordinates by a \ac{spd} matrix $\bm G(\bm x)\!\in\!\SPDn$, yielding the Riemannian inner product 
    $\langle \bm v, \bm w\rangle_{\bm G(\bm x)} = \bm v^\intercal \bm G(\bm x) \bm w$,
with $\bm x \!\in\! \mathcalm$ and ${\bm v, \bm w \!\in\! \tangentm{\bm{x}}}$.
The Riemannian metric defines distances, angles, and geodesics on $\mathcalm$, see~\cite{Lee13:SmoothManifolds, Lee18:RiemannianManifolds} for in-depth introductions. 

A smooth map $\varphi_\mathcalm \!:\! \checkmathcalm \!\to\! \mathcalm$ between two smooth manifolds $\checkmathcalm$ and $\mathcalm$ with $\text{dim}(\checkmathcalm)\!=\!d$  is a smooth immersion if its differential 
$d\varphi_\mathcalm\vert_{\checkbmx}$ is injective at each $\checkbmx\in\check{\mathcalm}$. A smooth embedding is a smooth immersion that is also a homeomorphism onto its image $\varphi (\checkmathcalm) \subseteq \mathcalm$. Intuitively, an embedding is a one-to-one continuous mapping between the manifold $\checkmathcalm$ and a $d$-dimensional subset $\varphi (\checkmathcalm)$ of $\mathcalm$, called an embedded submanifold of $\mathcalm$.
The lifted embedding $\varphi:\mathcalt\checkmathcalm \to \tangentbundlem$ is the embedding constructed as 
\begin{equation}
    \varphi(\checkbmx, \dcheckbmx) = (\varphi_\mathcalm(\checkbmx), \; d\varphi_\mathcalm\vert_{\checkbmx}\dcheckbmx).
\end{equation}

A smooth map $\xi_\mathcalm \!:\! \mathcalm \!\to\! \checkmathcalm$ between two smooth manifolds is a smooth submersion if its differential $d\xi_\mathcalm\vert_{\bm{x}}$ is surjective at each $\bm{x}\!\in\!\mathcalm$. A Riemannian submersion decomposes the tangent space as $\mathcal{T}_{\bm{x}}\mathcalm = \mathcal{V}_{\bm x} + \mathcal{H}_{\bm x}$. The vertical subspace $\mathcal{V}_{\bm x} = \text{ker}(d\xi_\mathcalm\vert_{\bm{x}})$ is the kernel of $d\xi_\mathcalm\vert_{\bm{x}}$, i.e., the part of $\mathcal{T}_{\bm{x}}\mathcalm$ that gets mapped to $\bm{0}\in\mathcal{T}_{\xi_\mathcalm(\bm{x})}\checkmathcalm$. The horizontal subspace $\mathcal{H}_{\bm x}=\mathcal{V}_{\bm x}^\perp$ is its orthogonal complement with respect to the metric $\bm{G}$ such that for any $\bm v, \bm w \in \tangentm{\bm{x}}$
\begin{equation}
    \langle \bm v, \bm w\rangle_{\bm G(\bm x)} = \langle d\xi_\mathcalm\vert_{\bm{x}} \bm v, d\xi_\mathcalm\vert_{\bm{x}} \bm w \rangle_{\check{\bm{G}} (\xi_{\mathcalm} (\bm x))},
\end{equation}
i.e., $d\xi_\mathcalm\vert_{\bm{x}}$ is a linear isometry from $\mathcal{H}_{\bm x}$ to $\mathcal{T}_{\xi_\mathcalm(\bm{x})}\checkmathcalm$. 
\vspace{-0.05cm}
\subsection{Riemannian Geometry of Lagrangian Systems}
The configurations $\bm q$ of an $n$-\ac{dof} system live on an $n$-dimensional smooth manifold $\mathcal{Q}\subseteq\mathbb R^n$ with a global chart~\cite{AbrahamMarsden87:Foundations,BulloLewis}. At each $\bm q\in\mathcal Q$, generalized velocities are tangent vectors $\dbmq\in\mathcal T_{\bm q} \mathcal Q$.
Given a time interval $\mathcal{I} = \left[t_0, t_\text{f}\right]$, system trajectories $\traj: \mathcal{I} \to \tangentbundleq: t \mapsto \left(\bmq(t), \dot{\bmq}(t)\right)^\intercal$ evolve on the tangent bundle $\tangentbundleq$. 
A Lagrangian system is a tuple $(\mathcalq, \lagrangian)$ composed of a configuration manifold $\mathcalq$ and a time-independent Lagrangian $\lagrangian\!:\! T\mathcalq \to \mathbb{R}$ defined as the difference between the kinetic and potential energies as  $\lagrangian(\bmq, \dbmq) \!=\! \frac{1}{2}\dbmq^{\intercal} \massmatrix \dbmq - \potentialenergy$.
The quadratic structure of the kinetic energy induces a Riemannian metric on $\mathcalq$, which, in coordinates, corresponds to the system's mass-inertia matrix $\bm M(\bm q)\in\SPDn$.
The equations of motion are derived from the principle of least action and are given, in coordinates, by
\begin{equation}
\label{eq:equationsofmotion}
    \massmatrix\ddbmq + \coriolistermmatrixform + \gravityterm = \bmtau,
\end{equation}
where $\bm{C}$ with elements $C_{ij} \!=\! \frac{1}{2}\sum_{k}\big(\frac{\partial M_{ij}}{\partial q^k}+\frac{\partial M_{ik}}{\partial q^j}-\frac{\partial M_{jk}}{\partial q^i}\big)\dot{q}_k$,  
represents the influence of Coriolis and centrifugal forces, $\gravityterm = \dpartial{\potentialenergy}{\bmq}$, and $\bmtau\in\cotangentq{\bmq}$ is a covector of generalized forces. 
Trajectories $\traj(t)$ are obtained by solving~\eqref{eq:equationsofmotion} recursively from an initial value $\traj(t_0) \!=\! \left(\bmq(t_0), \dot{\bmq}(t_0)\right)^\intercal$.

We model the generalized forces $\bmtau$ as a combination of dissipative and control forces.
We consider a Rayleigh dissipative potential $\mathcal{D}(\bmq, \dot{\bmq}) = \frac{1}{2} \dot{\bmq}^\intercal \bm{D}(\bmq)\dot{\bmq}$ with positive semi-definite dissipation matrix $\bm{D}(\bmq)$, leading to dissipative forces
$\bmtau_{\mathcal{D}} = -\frac{\partial \mathcal{D}(\bmq, \dot{\bmq})}{\partial \dot{\bmq}} = -\bm{D}(\bmq)\dot{\bmq}$.
Moreover, we model the set of admissible control torques as
$\bm \tau_{\text{c}} = \bm B(\bm q)\bm u$, where the control inputs $\bm u \in \mathbb R^m$ are restricted onto a subspace defined by the full-column-rank actuation distribution ${\bm B\!:\! \mathcalq \!\to\! \mathbb R^{n\times m}}$.
The system is fully actuated when $m\!=\!n$ and underactuated when $m \!<\! n$. 
The resulting generalized forces are\looseness-1
\begin{equation}
    \bmtau = \bmtau_{\mathcal{D}} + \bmtau_{\text{c}} = -\bm D(\bm q)\dbmq + \bm B(\bm q)\bm u.
    \label{eq:generalized_forces}
\end{equation}

\subsection{Projection-Based Lagrangian Reduced-Order Models}
Given the known parametrized governing equations of a high-dimensional system, referred to as the \ac{fom}, intrusive \ac{mor} aims to learn a low-dimensional surrogate dynamic model, i.e., a \ac{rom}, that accurately and computationally efficiently approximates the \ac{fom} trajectories~\cite{Schilders2008MOR}. 
Structure-preserving \ac{mor} uses (non)linear projections to maintain the structure of the \ac{fom} within the \ac{rom}, thereby preserving its energy conservation laws and stability properties~\cite{Marsden03:StructurePresMOR,Buchfink2024MOR}. 

For a Lagrangian system $(\mathcalq, \lagrangian)$, structure-preserving \ac{mor} infers a reduced Lagrangian system $(\check \mathcalq, \check \lagrangian)$ with $\text{dim}(\checkmathcalq)=d \ll n$. The reduced configuration manifold $\check \mathcalq$ is defined via the lifted embedding $\varphi\!:\! \mathcal T\check{\mathcal Q} \!\to\! \tangentbundleq$ of a smooth embedding $\varphiq\!:\! \checkmathcalq \!\to\! \mathcalq$, which is learned such that the embedded tangent bundle $\varphi(\mathcal T\check{\mathcal Q})\subseteq \tangentbundleq$ well captures the set of all trajectories
\begin{equation}
\label{eq:set_of_all_solutions}
    S = \left\{\traj(t) \in \mathcal T \mathcal Q \:\vert\: t \in \mathcal{I} \right\}\subseteq\mathcal T \mathcal Q.
\end{equation}
The reduced Lagrangian is constructed as the pullback of $\lagrangian$ by the lifted embedding $\varphi$, i.e., $\check{\lagrangian} = \varphi^*\lagrangian = \lagrangian \circ \varphi$. Direct construction of the reduced Euler-Lagrange equations 
yields
\begin{equation}
    \label{eq:reducedequationsofmotion}
    \check{\bmM}(\checkbmq)\ddcheckbmq + \checkbmC(\checkbmq, \dcheckbmq)\dcheckbmq + \check{\bm{g}}(\checkbmq) = \check{\bmtau}.
\end{equation} 
Intrusive \ac{mor}~\cite{Buchfink2024MOR} computes the reduced parameters from the known high-dimensional dynamics by pulling them back as 
\begin{equation}
    \label{eq:pullback_dynamic_parameters}
    \begin{aligned}
    &\check{\bmM}(\checkbmq) = d\varphiq\vert_{\checkbmq}^\intercal \; \massmatrix \; d\varphiq\vert_{\checkbmq}, \\
    \check{\bm{g}}(\checkbmq) &= d\varphiq\vert_{\checkbmq}^\intercal \; \gravityterm, \quad \text{and} \quad \check{\bmtau}=d\varphiq\vert_{\checkbmq}^\intercal \; \bmtau.
    \end{aligned}
\end{equation} 
Note that $\checkbmC(\checkbmq, \dcheckbmq)$ is computed via the derivatives of $\check{\bmM}(\checkbmq)$ and thus already accounts for embedding-induced curvature forces. 
When considering dissipative forces as in~\eqref{eq:generalized_forces}, the dissipative potential is pulled back similarly to the Lagrangian, yielding a reduced damping force $\check{\bmtau}_{\mathcal{D}} \!=\! -\frac{\partial \check{\mathcal{D}}(\check{\bmq}, \dot{\check{\bmq}})}{\partial \dot{\check{\bmq}}} \!= \! -\check{\bm{D}}(\check{\bmq})\dot{\check{\bmq}}$ with positive semi-definite reduced dissipation matrix~\cite{friedl2025hamiltonian}
\begin{equation}
\label{eq:pullback_dissipation}
    \check{\bm D}(\check {\bm q}) = d\varphi_{\mathcal Q}\vert_{\checkbmq}^\intercal \; \bm D(\bm q) \; d\varphi_{\mathcal Q}\vert_{\checkbmq}.
\end{equation}

Reduced trajectories $\checktraj(t) \!=\! \left(\checkbmq(t), \dot{\checkbmq}(t)\right)^\intercal$ are obtained by recursively solving the reduced equations of motion~\eqref{eq:reducedequationsofmotion}. The reduced initial value $\checktraj(t_0) \!=\! \rho(\traj (t_0)) \!=\! \left(\checkbmq(t_0), \dot{\checkbmq}(t_0)\right)^\intercal$, required to solve~\eqref{eq:reducedequationsofmotion}, is computed from the \ac{fom} initial value via the reduction map $\rho(\bmq, \dbmq) \!=\! (\rhoq(\bmq), d\rhoq\vert_{\bmq}\dbmq)$ associated with $\varphi$ and defined as the pair of point and tangent reductions $\rhoq: \mathcalq \to \checkmathcalq$ and $d\rhoq\vert_{\bmq}: \mathcal{T}_{\bmq}\mathcalq \to \mathcal{T}_{\rhoq(\bmq)}\checkmathcalq$. Trajectories of the high-dimensional Lagrangian system are obtained as solutions $\traj(t) \approx \varphi(\checktraj(t))$.
Importantly, the preservation of the Lagrangian structure in the \ac{rom} guarantees the conservation of the energy $\check{\mathcal{E}} = \mathcal{E} \circ \varphi$ of unforced dynamics ($\bmtau=0$) along the trajectories $\checktraj(t)$ and their images $\varphi(\checktraj(t))$.

To ensure the consistency of the \ac{rom} with respect to the \ac{fom}, the composition $\projection \!=\! \varphi \circ \rho : \tangentbundleq \to \varphi(\mathcal T\check{\mathcal Q})$ must be a projection, i.e., an idempotent mapping $\projection\circ \projection \!=\! \projection$. Intuitively, this ensures that applying the reduction and the embedding on a state that is already contained in the embedded submanifold $\varphi(\mathcal T\check{\mathcal Q})$ returns exactly the same state. A non-idempotent mapping would instead move the state at each application, thus leading to physical inconsistencies.
The mapping $p$ is a valid projection if the two projection properties
\begin{equation}
\label{eq:projectionproperty}
    \rhoq \circ \varphiq = \identity{ \check{\mathcal Q}} \quad \text{ and } \quad 
    d\rho\vert_{\varphi(\checkbmq, \dcheckbmq)} \circ d\varphi\vert_{(\checkbmq, \dcheckbmq)} = \identity{\mathcal{T}_{\checkbmq} \check{\mathcal Q}},
\end{equation}
are satisfied $\forall (\checkbmq, \dcheckbmq) \!\in\! \mathcal T \check{\mathcal Q}$.
Trivially, a pair of point reduction and embedding fulfilling the former also satisfies the latter. 

\subsection{Reduced-order Lagrangian Neural Networks (RO-LNNs)}
\label{sec:background_learning_ro_lnn}
Intrusive \ac{mor} techniques assume known additional dynamics parameters which are, in practice, rarely available. In this paper, we focus on structure-preserving non-intrusive \ac{mor}, where the two components of the reduced system $(\check{\mathcal Q}, \check{\mathcal L})$ are learned from data. 
Specifically, we focus on the \ac{rolnn}~\cite{friedl2025reduced} as an expressive physics-inspired nonlinear architecture based on Riemannian geometry.
A \ac{rolnn} jointly learns a nonlinear embedded submanifold $\varphi(\mathcal T\check{\mathcal Q})$ using a constrained \ac{ae} and the associated low-dimensional Lagrangian dynamics with a \ac{lnn}.
First, the \ac{rolnn} parametrizes the point reduction $\rhoq$ and embedding $\varphiq$ as the encoder and decoder of a constrained \ac{ae} with biorthogonal weights, opposite biases, and invertible activation functions~\cite{Otto2023MOR}. 
The lifted reduction map and embedding are constructed via the Jacobians of the encoder and decoder. 
This architecture fulfills the projection properties~\eqref{eq:projectionproperty} by construction, i.e., the lifted constrained \ac{ae} mapping $\liftedPiAE\!:\!\mathcal T\mathcal Q \!\to\! \varphi(\mathcal T{\check{\mathcal Q}})$ is a projection.
Second, the \ac{rolnn} learns a reduced Lagrangian $\check{\lagrangian}$ via a latent geometric \ac{lnn}~\cite{friedl2025reduced}. Specifically, the reduced mass-inertia matrix and the reduced potential energy are parametrized as a \ac{spd} network $\check{\bm M}_{\bm \theta_{\check T}}\!:\!\checkmathcalq\!\to\!\SPDd$ and a \ac{mlp} $\check V_{\bm \theta_{\check V}}\!:\!\checkmathcalq \!\to\! \euclideanspace$, respectively. Importantly, the \ac{spd} network accounts for the Riemannian geometry of the \ac{spd} manifold, thus overcoming the geometric flaws of 
\acp{lnn} parametrizing the Cholesky decomposition $\bmM_\theta=\bm L_\theta \bm L_\theta^\intercal$~\cite{Lutter2023DeLaN}. As~\cite{friedl2025hamiltonian}, we parametrize the reduced dissipation matrix as a second \ac{spd} network $\check{\bm D}_{\theta_{\check D}}\!:\! \checkmathcalq \!\to\! \SPDd$.

Given sets of observations $\left\{\bmq_i(\mathcal{I}_i), \dbmq_i(\mathcal{I}_i), \bm{\tau}_i(\mathcal{I}_i)\right\}_{i=1}^N$ over intervals $\mathcal{I}_i = \left[t_{i}, t_{i}+H\Delta t\right]$ with constant integration time $\Delta t$, the \ac{rolnn} is trained to minimize the multi-step loss

\footnotesize
\vspace{-0.5cm}
\begin{align} 
    \label{eq:loss_RO_LNN}
    \ell_{\text{RO-LNN}}
    &= \frac{1}{HN} \sum_{i=1}^N \sum_{j=1}^H 
    \underbrace{\|\tilde{\bmq}_{i}(t_{i,j}) - \bmq_i(t_{i,j}) \|^2 + \| \dtildebmq_{i}(t_{i,j}) - \dbmq(t_{i,j}) \|^2}_{\ell_{\text{AE}}} \nonumber \\
    &+  
    \underbrace{ \|\dcheckbmq_{\text{p}, i}(t_{i,j}) - \dcheckbmq_i(t_{i,j}) \|^2}_{\ell_{\text{LNN},d}} + \underbrace{\|\dtildebmq_{\text{p}, i}(t_{i,j}) - \dbmq_{i}(t_{i,j})\|^2}_{\ell_{\text{LNN},n}} + \; w\; \|\bm{\theta}\|_2^2, 
    \vspace{-0.2cm}
\end{align}
\normalsize
where $(\tildebmq, \dtildebmq) \!=\! \liftedPiAE(\bmq, \dbmq)$, $\dcheckbmq_{\text{p}, i}(t_{i,j}) \!=\! \int_{t_{i}}^{t_{i,j}}  \ddcheckbmq_{\text{p},i}  \mathrm{dt}$ with ${t_{i,j}\!=\!t_{i} + j\Delta t}$, and $\dtildebmq_{\text{p}, i}(t_{i,j}) =d\varphiq|_{\checkbmq_{\text{p}, i}(t_{i,j})} \dcheckbmq_{\text{p}, i}(t_{i,j})$.  
The network parameters belonging to the biorthogonal and \ac{spd} manifold are optimized via Riemannian optimization~\cite{Absil07:RiemannOpt, Boumal22:RiemannOpt,becigneul2018riemannianoptimization}.

\section{Model-based Control via RO-LNNs}
\label{sec:method_part1}

We consider a high-dimensional Lagrangian system following the dynamics~\eqref{eq:equationsofmotion}. Our objective is to track a given smooth, time-dependent, bounded reference trajectory $(\bmqdes(t), \dbmqdes(t), \ddbmqdes(t)) \in \mathcal T \mathcal T{\mathcal{Q}}$. 
Throughout this section, we assume fully-actuated systems with $\bm B(\bm q)\!=\!\mathbf I$ and $\bm\tau_{\text{c}} \!=\! \bm u$. We discuss the indirectly- and underactuated cases in Sec.~\ref{sec:method_part2}.
The closed-loop dynamics from~\eqref{eq:equationsofmotion}-\eqref{eq:generalized_forces} can be rewritten in terms of errors $\bm e = \bmq - \bmqdes$ and for a general control torque $\bm\tau_{\text{c}}$ as\looseness-1
\begin{equation}
\vspace{-0.05cm}
    \label{eq:general_cl_dyns}
    \begin{aligned}
            &\bm M(\bm q)\ddbme + (\bm C (\bm q, \dbmq) + \bm D(\bm q))\dbme = \bm\tau_{\text{c}} \\&- (\bm M(\bm q)\ddbmqdes + (\bm C (\bm q, \dbmq) + \bm D(\bm q))\dbmqdes + \bm g(\bm q)).
    \end{aligned}
\end{equation}
Setting the control torque as a standard model-free \ac{pd} controller
$\bm\tau_{\text{c}} \!=\! \bmtau_{\text{PD}} \!=\! - \Kp\bm e - \Kd\dot{\bm e}$ with gain matrices $\Kp, \Kd \in \SPDn$ only guarantees set-point regulation with a steady-state error as it does not compensate for dynamics. Moreover, the reference dynamics act as a disturbance when tracking time-varying trajectories, and tracking accuracy degrades significantly.
Explicit compensation of the dynamics is required to achieve zero steady-state error and allows exact tracking when $\dbmqdes, \ddbmqdes \neq \mathbf 0$. This can be achieved via the globally asymptotically stable \ac{pd}+ control law~\cite{padenpanja88}~---~extendable to exponential results~\cite{murray1994,wenbayard88}~---~that includes a current-state-based feedforward torque $\bmtau_{\text{FF}}$ as

\small
\vspace{-0.45cm}
\begin{equation}
    \label{eq:pd_ff_controllaw}
    \bm\tau_{\text{c}} = \underbrace{\bm M(\bm q)\ddbmqdes + \bm (\bm C(\bmq, \dbmq)  + \bm D(\bm q)) \dbmqdes+ \bm g(\bm q)}_{\bmtau_{\text{FF}}} \underbrace{- \Kp\bm e - \Kd\dot{\bm e}}_{\bmtau_{\text{\ac{pd}}}}.
\end{equation}
\normalsize
However, in practice, any guarantees are undermined by modeling errors and disturbances that restrict convergence to local neighborhoods and to bounded errors~\cite{khalil2002}. This is particularly critical for high-dimensional systems for which access to a \ac{fom} is often unrealistic.\looseness-1

In this paper, we consider high-dimensional Lagrangian systems whose unknown dynamics are learned via a \ac{rolnn}. We assume that the system's trajectories can be well approximated by a \ac{rom} such that $S$ in~\eqref{eq:set_of_all_solutions} is approximately a subset of $\varphi(\mathcal{T}\checkmathcalq)$ for some embedding $\varphi$ and that the \ac{rolnn} achieves a low loss~\eqref{eq:loss_RO_LNN}. 
In Sec.~\ref{sec:method_part1_control}, we formulate a \ac{pd}+ control law in the \ac{rolnn} latent space with a feedforward action based on the learned reduced dynamic parameters $\checkbmMtheta$, $\checkbmgtheta$, and  $\checkbmDtheta$. 
In Sec.~\ref{sec:method_part1_latent_convergence_analysis}, we analyze convergence of the proposed controller in the \ac{rolnn} latent space as a function of the two sources of modeling errors arising in this scenario: (1) the modeling error of the learned reduced dynamic parameters, and (2) the projection alignment error of the learned \ac{ae}. 
Finally, Sec.~\ref{sec:method_part1_fullorder_plant_convergence_analysis} analyzes convergence of the proposed controller on the embedded submanifold $\varphiq({\checkmathcalq})\!\subseteq \!\mathcalq$ and subsequently on the complete configuration manifold $\mathcalq$. 

\vspace{-0.05cm}
\subsection{A RO-LNN-based Control Law for Trajectory Tracking}
\vspace{-0.05cm}
\label{sec:method_part1_control}
We design a trajectory tracking controller that leverages the Lagrangian structure inherent in both the original system and the structure-preserving \ac{rom} to establish performance guarantees. 
Specifically, we exploit state measurement feedback and the learned \ac{rolnn} to define a latent control law $\checkbmtau_{\text{c}}$ as combination of a latent feedforward $\check{\bmtau}_{\text{FF}}$ and latent \ac{pd} feedback $\check{\bmtau}_{\text{PD}}$. 
First, we obtain the latent reference trajectory $(\checkbmqdes(t), \dcheckbmqdes(t), \ddcheckbmqdes(t)) \in \mathcal T \mathcal T\check{\mathcal{Q}}$ by encoding the reference trajectory into the learned latent space with $\checkbmqdes = \rhoq(\bmqdes)$, $\dcheckbmqdes = d\rhoq \vert_{\bmqdes} \dbmqdes$, and $\ddcheckbmqdes =  d\rhoq\vert_{\bmqdes} \ddbmqdes + d^2\rhoq\vert_{\bmqdes} \dbmqdes^2$. We then formulate a \ac{rolnn}-based \ac{pd}+ control law as

\small
\vspace{-0.4cm}
\begin{equation}
    \label{eq:reduced_pd_ff_controllaw}
    \checkbmtau_{\text{c}} = \underbrace{\check{\bm M}_{\bm\theta}(\checkbmq)\ddcheckbmqdes + (\check{\bm C}_{\bm\theta}(\checkbmq, \dcheckbmq) + \check{\bm D}_{\bm\theta}(\checkbmq))\dcheckbmqdes + \check{\bm g}_{\bm\theta}(\checkbmq)}_{\checkbmtau_{\text{FF}}} \underbrace{ - \check{\bm K}_{\text{p}}\check{\bm e} - \check{\bm K}_{\text{d}} \dot{\check{\bm e}}}_{\checkbmtau_{\text{\ac{pd}}}},
\end{equation}
\normalsize
with latent tracking error $\checkbme = \checkbmq - \checkbmqdes$ and reduced gain matrices $\checkKp, \checkKd \in \SPDd$.
The reduced control torques $\checkbmtau_{\text{c}}$ are finally lifted to the original representation via the encoder pullback,
\begin{equation}
\label{eq:torque_reconstruction}
    \tildebmtau_{\text{c}} = d\rhoq\vert_{\bmq}^\intercal \check{\bmtau}_{\text{c}},
\end{equation}
before being applied on the original system by setting $\bmtau_{\text{c}} = \tilde{\bm \tau}_{\text{c}}$. 
It is important to observe that this formulation restricts the possible range of control torques to $\tildebmtau_{\text{c}} \in \mathcal T_{\tildebmq}^*\varphiq(\mathcal{Q})$, so that the system evolves according to the closed-loop dynamics
\begin{equation}
\label{eq:plant_actuated_via_reconstructed_torque}
    \bm M(\bm q)\ddbmq + (\bm C(\bm q, \dbmq) + \bm D(\bm q))\dbmq + \bm g(\bm q) = \tildebmtau_{\text{c}}.
\end{equation}
Under exact modeling assumptions, i.e., a \ac{rolnn} with zero loss~\eqref{eq:loss_RO_LNN} $\forall (\bmq,\dbmq)\in\mathcal{T}\mathcalq$, the controller achieves perfect convergence, as shown for the latent convergence of regulation tasks in~\cite{Lepri2024MOR}. However, in realistic scenarios, both the \ac{ae} and the latent \ac{lnn} inevitably contribute to modeling errors. Therefore, we analyze the properties of the proposed control law under the significantly milder assumption of a low loss~\eqref{eq:loss_RO_LNN}.
\begin{figure}
    \centering
    \adjustbox{trim=0.0cm 0.2cm 0.0cm 0.0cm}{
            \includesvg[width=.99\linewidth]{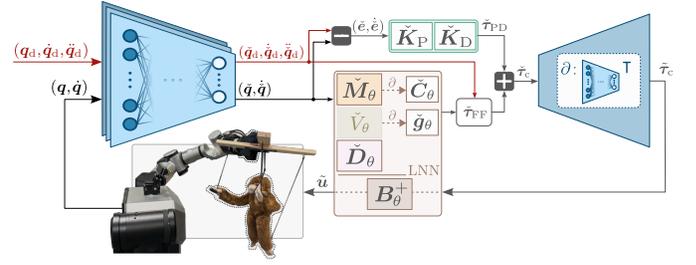}  
        }
    \caption{\ac{rolnn}-based control loop. The AE is depicted in blue and the latent \ac{lnn} components are depicted within the brown box. The control torques are computed in the learned latent space and lifted through the decoder.
    }
    \label{fig:closedloop}
    \vspace{-0.55cm}
\end{figure}

\vspace{-0.1cm}
\subsection{Latent Space Analysis of Stability and Convergence}
\label{sec:method_part1_latent_convergence_analysis}
To establish the stability and convergence properties of the \ac{rolnn}-based \ac{pd}+ law~\eqref{eq:torque_reconstruction}, we first examine the closed-loop dynamics within the learned Lagrangian latent space. Our analysis leverages that the \ac{rolnn} satisfies the projection properties~\eqref{eq:projectionproperty}, learns \ac{spd} mass-inertia and damping matrices $\checkbmMtheta$, $\checkbmDtheta$, and fulfills ${(\dot{\check{\bm M}}_\theta -2 \checkbmC_\theta)=\bm 0}$ by construction.

\subsubsection{Closed-loop dynamics} 
In latent space, the nominal reduced system follows the Lagrangian dynamics~\eqref{eq:reducedequationsofmotion} with pulled-back parameters~\eqref{eq:pullback_dynamic_parameters}-\eqref{eq:pullback_dissipation}.
The \ac{rolnn} identifies the nominal latent dynamic equations exactly only if the loss~\eqref{eq:loss_RO_LNN} is $0$ $\forall (\bmq,\dbmq)\in\mathcal{T}\mathcalq$.
In practice, two types of modeling errors arise. First, the \emph{dynamic modeling error} is due to the difference between the nominal $\{\checkbmM, \checkbmg, \checkbmD\}$~\eqref{eq:pullback_dynamic_parameters}-\eqref{eq:pullback_dissipation} and learned reduced parameters $\{\checkbmMtheta, \checkbmgtheta, \checkbmDtheta\}$.
Second, the \emph{projection alignment error} arises from the learned \ac{ae} as explained next.

While the constrained \ac{ae} used in the \ac{rolnn} is guaranteed to learn a projection~\cite{Otto2023MOR}, the Riemannian angle between the subspaces $\mathcal{V}_{\bm q} \!=\! \text{ker}(d\rhoq \vert_{\bm{q}})$ and ${\mathcal{H}_{\bm q} \!= \! \text{im}(d\rhoq \vert_{\bm{q}})}$ is subject to parameter convergence. Concretely, this means that the Riemannian angle $\alpha$ between the projected and nullspace components of a velocity $\dbmq\in\mathcal{T}_{\bm q}\mathcalq$ can range as 
\begin{equation}
    \label{eq:riemannian_angle_between_image_and_nullspace}
        \alpha = \arccos \big(\sfrac{|\langle\dtildebmq, \dbmqns\rangle_{\bm M} |}{ \|\dtildebmq\|_{\bm M} \|\dbmqns\|_{\bm M} } \big) \in[0,\sfrac{\pi}{2}], 
\end{equation}
where we defined $\dtildebmq \!=\! \bm{P}(\bmq) \dbmq \in \mathcal{H}_{\bm q}$ with the projection $\bm{P}(\bmq) \!=\! d\varphiq\vert_{\checkbmq} d\rhoq\vert_{\bmq}$, and $\dbmqns \!=\! \dbmq - \dtildebmq \in\mathcal{V}_{\bm q}$, see Fig.~\ref{fig:submersion_angles}.
However, any ${\alpha\!\neq\!\sfrac{\pi}{2}}$ induces a cross-coupling between force components in the nullspace $\mathcal{V}_{\bmq}^*$ and submanifold $\mathcal{H}_{\bmq}^*$, resulting in additional non-conservative latent forces.
As a consequence, the \ac{rom} is fully dynamically consistent~---~i.e., the latent trajectories obtained from the reduced dynamic equations~\eqref{eq:reducedequationsofmotion} are equal to the encoding of the trajectories following~\eqref{eq:equationsofmotion}~---~only if $\alpha\!=\!\sfrac{\pi}{2}$~\cite[Thm. 5.7]{Buchfink2024MOR}. 
In this case, the encoder $\rhoq$ defining a Riemannian submersion and $d\rhoq$ is the dynamically-consistent generalized inverse of $d\varphiq$, i.e., $d\rhoq \!=\! d\varphiq^{+_{,\bm M}}$ with 
\begin{equation}
\label{eq:riemannian_pseudoinverse_of_embedding}
    d\varphiq\vert^{+_{,\bm M}}_{\checkbmq} = (d\varphiq\vert_{\checkbmq}^\intercal \; \bm M(\bmq) \; d\varphiq\vert_{\checkbmq})^{-1} \; d\varphiq\vert_{\checkbmq}^\intercal \; \bm M(\bm q),
\end{equation}
and the projection $\bm{P}_{\perp}(\tildebmq) \!=\! d\varphiq\vert_{\checkbmq} d\varphiq\vert_{\checkbmq}^{+_{,\bmM}}$ is $\bmM$-orthogonal, as it yields orthogonal subspaces $\mathcal{H}_{\bm q} \!=\! \mathcal{V}_{\bm q}^\perp$. 
Notice that $d\varphiq^{+_{,\bm M}}$ cannot be computed in practice as $\bmM$ remains unknown. Moreover, $\varphiq$ can be identified from $d\varphiq$ only for linear embeddings~\cite{Carlberg2015LagrMOR}. 
In summary, the projection alignment error arises from the difference between the learned ${\bm{P}(\bmq) \!=\! d\varphiq\vert_{\checkbmq} d\rhoq\vert_{\bmq}}$ and the nominal $\bmM$-orthogonal $\bm{P}_{\perp}$. 

Next, we write the latent closed-loop dynamics defined on a compact neighborhood $\mathcal{N}\subseteq\mathcal{T}\mathcalq$ close to the training distribution where all system matrices are well-behaved.

\begin{figure}
    \centering
    \adjustbox{trim=0.0cm 0.3cm 0.0cm 0.0cm}{
            \includesvg[width=.55\linewidth]{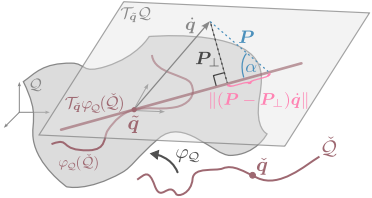}  
        }
    \caption{Riemannian residual angle $(\SI{90}{\degree}-\alpha)$ between the $\bmM$-orthogonal ($\bm P_\perp$) and learned ($\bm P$) projections. 
    }
    \label{fig:submersion_angles}
    \vspace{-0.55cm}
\end{figure}

\vspace{-0.1cm}
\begin{mdframed}[hidealllines=true,backgroundcolor=lightlightgray,innerleftmargin=.1cm,
  innerrightmargin=.1cm,innertopmargin=-.1cm,innerbottommargin=0.1cm]
\begin{proposition}[Latent dynamics]
    \label{prop:latent-closedloop}
    The latent closed-loop dynamics induced by the control torque $\tildebmtau_{\text{c}}$~\eqref{eq:torque_reconstruction} take the form
    \begin{equation}
        \label{eq:latent_closed_loop_dyns}
        \checkbmM(\checkbmq) \ddcheckbme + \big(\checkbmC(\checkbmq, \dcheckbmq) + \checkbmD(\checkbmq) + \checkKd\big)\dcheckbme + \checkKp \checkbme = \check{\Delta}_{\theta} + \check{\Delta}_{\perp},
    \end{equation}
    where $\check{\Delta}_{\theta}$ and $\check{\Delta}_{\perp}$ are disturbances due to dynamic modeling and projection alignment errors, that are bounded in a compact neighborhood $\mathcal{N}\!\subseteq\!\mathcal{T}\mathcalq$ of the training data distribution.
\end{proposition}
\end{mdframed}
\vspace{-0.2cm}
\begin{proof}[Proof sketch]
We quantify the difference between the nominal latent acceleration $\ddcheckbmq_{\text{nom}}$ induced directly from the latent control torque $\checkbmtau_{\text{c}}$ in~\eqref{eq:reducedequationsofmotion} and the latent acceleration $\ddcheckbmq$ encoding the acceleration $\ddbmq$ obtained as the response of the full system to the lifted control torques $\tildebmtau_{\text{c}}$~\eqref{eq:torque_reconstruction}.
Dropping system matrices dependencies, the difference can be expanded as
\begin{equation}
\label{eq:diff_latent_accelerations}
    \checkbmM\ddcheckbmq = \checkbmM\ddcheckbmq_{\text{nom}} + \checkbmM(\ddcheckbmq - \ddcheckbmq_{\text{nom}}),
\end{equation}
Substituting ${\ddcheckbmq \!=\! d\rhoq\vert_{\bmq}\ddbmq +d^2\rhoq\vert_{\bmq} \dbmq^2}$, $\ddbmq$ from~\eqref{eq:equationsofmotion} with $\bmtau \!=\! \tildebmtau_{\text{c}}$ from~\eqref{eq:torque_reconstruction}, $\ddcheckbmq_{\text{nom}}$ from~\eqref{eq:reducedequationsofmotion} with $\checkbmtau\!=\!\checkbmtau_{\text{c}}$ from~\eqref{eq:reduced_pd_ff_controllaw}, $\ddcheckbmq\!=\! \ddcheckbme + \ddcheckbmqdes$ and $\dcheckbmq\!=\! \dcheckbme + \dcheckbmqdes$, we rewrite~\eqref{eq:diff_latent_accelerations} as the latent-closed loop dynamics~\eqref{eq:latent_closed_loop_dyns}. The complete derivation is provided in App.~\ref{appendix:latent_closedloop_dynamics}. Next, we describe and bound the disturbances $\check{\Delta}_{\theta}$ and $\check{\Delta}_{\perp}$. 

The dynamic modeling disturbance arises from the error between the nominal and learned reduced dynamic parameters as \looseness-1
\begin{equation}
\label{eq:delta_check_parameter_convergence_disturbance}
    \check{\Delta}_{\theta} = (\checkbmMtheta - \checkbmM)\ddcheckbmqdes + (\checkbmCtheta - \checkbmC + \checkbmDtheta - \checkbmD)\dcheckbmqdes + (\checkbmgtheta - \checkbmg),
\end{equation}
equivalently $\check{\Delta}_{\theta}\!=\!\check{\bm\tau}_{\text{FF}} - \check{\bm\tau}_{\text{nom-FF}}$
with $\check{\bm\tau}_{\text{FF}}$ and $\check{\bm\tau}_{\text{nom-FF}}$ denoting the feedforward torques of the control law~\eqref{eq:fom_pd_lat_ff_controllaw} computed using the learned and nominal reduced dynamic parameters, respectively.
Assuming a bounded reference trajectory and low losses $\ell_{\text{LNN},d}$ and $\ell_{\text{LNN},N}$ in~\eqref{eq:loss_RO_LNN}, it is straightforward to show that there exists a compact neighborhood $\mathcal{N}\subseteq\mathcal{T}\mathcalq$ of the training data distribution such that $\|\check\Delta_\theta\|\leq \check r_{\theta}$ for some constant $\check r_{\theta}$. Note that~\cite{liu2024lnns} used a similar argument for bounding the modeling error of full-order \acp{lnn}.

The projection alignment disturbance captures the error between the nominal $\bmM$-orthogonal and learned projections as \looseness-1

\small
\vspace{-0.4cm}
\begin{align}
\label{eq:delta_perp_orthogonality_convergence_disturbance}
    \check{\Delta}_{\perp} 
    &= (\checkbmM d\rhoq \bm M^{-1} d\rhoq^\intercal - \mathbf I)\checkbmtaudes + d\varphiq^\intercal \bm M(\mathbf I -d\varphiq d\rhoq)d^2\varphiq \dcheckbmq^2 \nonumber \\
    &\quad\quad+ (d\varphiq^\intercal -\checkbmM d\rhoq \bm M^{-1})((\bm C + \bm D)\dbmq + \bm g ),
\end{align}
\normalsize
equivalently $\check{\Delta}_{\perp} \!=\! \checkbmM(\ddcheckbmq - \ddcheckbmq_{\text{nom}})$.
It is straightforward to show that $\check\Delta_\perp\!=\!0$ when the \ac{ae} learns a Riemannian submersion, i.e., when $d\rhoq \!=\! d\varphiq^{+_{,\bm M}}$, see App.~\ref{appendix:latent_convergence_bounds}.
Otherwise, the norm $\|\check{\Delta}_{\perp}\|$ is bounded by the system state and the angle $\sfrac{\pi}{2}-\alpha$ between the nominal $\bmM$-orthogonal and learned projections $\bm{P}_\perp$ and $\bm{P}$
, see Fig~\ref{fig:submersion_angles}. 
As~\eqref{eq:delta_perp_orthogonality_convergence_disturbance} is the sum of three bounded terms in $\mathcal{N}$, we obtain ${\|\check{\Delta}_{\perp}\| \leq \|\check{\Delta}_{\perp,1}\| + \|\check{\Delta}_{\perp,2}\| + \|\check{\Delta}_{\perp,3}\|\leq \check r_{\perp}}$. 
The complete derivation of $\check r_{\perp}$ is provided in App.~\ref{appendix:latent_convergence_bounds}. 
\end{proof}

\vspace{-0.1cm}
\begin{mdframed}[hidealllines=true,backgroundcolor=lightlightgray,innerleftmargin=.1cm,
  innerrightmargin=.1cm,innertopmargin=-.1cm,innerbottommargin=0.1cm]
\begin{corollary}
\label{cor:bounded_delta_check}
    The latent closed-loop dynamics~\eqref{eq:latent_closed_loop_dyns} are such that, for states in the neighborhood $\mathcal{N}\subseteq\mathcal{T}\mathcalq$, there exists a constant $\bar d_{\check\Delta} \geq 0$ such that $\sup_{(\bm q, \dbmq)\in \mathcal N}\|\check\Delta_\theta + \check\Delta_\perp\| \leq \bar d_{\check\Delta}$.
\end{corollary}
\end{mdframed}
\vspace{-0.2cm}
\begin{proof}
    From Proposition~\ref{prop:latent-closedloop}
    and the triangle inequality, it follows that  
    $\|\check\Delta_\theta + \check\Delta_\perp\| \leq \|\check\Delta_\theta\| + \|\check\Delta_\perp\| \leq \check r_{\theta} + \check r_{\perp} =\bar d_{\check\Delta}$.
\end{proof}

We conclude with a corollary that discusses convergence of the latent error state $\check{\bm x} = (\checkbme^\intercal, \dcheckbme^\intercal )^\intercal$ in the nominal scenario and a remark on the \ac{rolnn} training. For a given bounded reference, $\Omega_{\mathcal N}$ denotes the neighborhood in the latent error state space where $\checkbmx$ remains small enough such that $(\bm q, \dbmq)\in \mathcal N$.\looseness-1

\vspace{-0.1cm}
\begin{mdframed}[hidealllines=true,backgroundcolor=lightlightgray,innerleftmargin=.1cm,
  innerrightmargin=.1cm,innertopmargin=-.1cm,innerbottommargin=0.1cm]
\begin{corollary}
\label{cor:latent_system_converges}
    The nominal latent closed-loop dynamics, i.e., the dynamics~\eqref{eq:latent_closed_loop_dyns} with $\check\Delta_\theta\!=\!\check\Delta_\perp \!=\!0$, are locally exponentially stable.
    Specifically, there exists a compact neighborhood $\Omega_0\!\subset\!\Omega_{\mathcal N}$ of the origin and constants $c \!>\! 0$, $\lambda \!>\! 0$ such that $\|\checkbmx(t)\|\leq ce^{-\lambda t}\|\checkbmx(0)\|$ holds $\forall\checkbmx(0)\in \Omega_0$.
\end{corollary}
\end{mdframed}
\vspace{-0.2cm}
\begin{proof}[Proof Sketch]
We consider the Lyapunov candidate ${\vartheta\!=\!\frac{1}{2} \bm x^\intercal \bm \Theta\bm x}$ with $\bm \Theta \!=\! \begin{psmallmatrix}\checkKp & \epsilon \checkbmM\\ \epsilon \checkbmM  &\checkbmM \end{psmallmatrix}$. For a sufficiently small $\epsilon$, we have $\bm \Theta\!\in\!\SPD$ and $v_1\|\check{\bm x}\|^2 \!\leq\! \vartheta(\check{\bm x}) \!\leq\! v_2 \|\check{\bm x}\|^2$ for constants $v_1, v_2\!>\!0$.
The time-derivative of $\vartheta$ is $\dot\vartheta \!=\! - \checkbmx \bm H \checkbmx$ with $\bm H\!=\!\begin{psmallmatrix}
\epsilon\checkKp & 0.5\epsilon(\checkKd + \checkbmD- \checkbmC) \\ 0.5\epsilon(\checkKd + \checkbmD- \checkbmC)^\intercal & (\checkKd + \checkbmD- \epsilon\check{\bm M})
\end{psmallmatrix}$.
From the positive definiteness conditions on the Schur complement, it follows that, for a sufficiently small $\epsilon$, $\bm H\in \SPD$ on $\Omega_{\mathcal N}$ and there exists a level set $\Omega_0 = \left\{\checkbmx\:|\:\vartheta(\checkbmx) \leq c_0\right\}\subset \Omega_{\mathcal N}$ such that, for any $\checkbmx(0)\in\Omega_0$, $\dot\vartheta\!<\!0$ ensures that the trajectory remains in $\Omega_0$.
Then, by the comparison lemma, there exists a constant $\lambda\!>\!0$ for which $\vartheta(t) \!\leq\! \vartheta(0)e^{-2\lambda t}$, and thus $\|\checkbmx(t)\| \!\leq\! c e^{-\lambda t}\|\checkbmx(0)\|$ with $c\!=\!\sqrt{\sfrac{v_2}{v_1}}$, see App.~\ref{appendix:latent_closed_loop_errorfree_proof} for details.\looseness-1
\end{proof}

\vspace{-0.1cm}
\begin{mdframed}[hidealllines=true,backgroundcolor=lightlightgray,innerleftmargin=.1cm,
  innerrightmargin=.1cm,innertopmargin=-.1cm,innerbottommargin=0.1cm]
\begin{remark}
    The loss~\eqref{eq:loss_RO_LNN} implicitly minimizes the dynamic modeling and projection alignment disturbances $\check\Delta_\theta$ and $\check\Delta_\perp$ as low losses $\ell_{\text{LNN}, d}$ and $\ell_{\text{LNN}, N}$ are only attainable when the dynamic modeling error is small and $\bm M$-orthogonality holds.
\end{remark}
\end{mdframed}
\vspace{-0.15cm}

\subsubsection{Stability and convergence analysis}
Next, we provide our main result on the stability and convergence of the latent error state evolving according to the latent closed-loop dynamics~\eqref{eq:latent_closed_loop_dyns} under robustness considerations.

\vspace{-0.15cm}
\begin{mdframed}[hidealllines=true,backgroundcolor=lightlightgray,innerleftmargin=.1cm,
  innerrightmargin=.1cm,innertopmargin=-.1cm,innerbottommargin=0.1cm]
\begin{theorem}[Latent stability and convergence]
    \label{thm:latentISS}
    The latent closed-loop dynamics~\eqref{eq:latent_closed_loop_dyns} under disturbances $\check\Delta_\theta$ and $\check\Delta_\perp$ are locally exponentially \ac{iss}, 
    i.e., there exist a compact neighborhood $\Omega_0\!\subset\!\Omega_{\mathcal N}$ of the origin and constants $c\!>\!0$, $\lambda\!>\!0$, $\gamma\!>\!0$ such that 
    ${\|\checkbmx(t)\|\leq ce^{-\lambda t}\|\checkbmx(t)\| + \gamma\|\check\Delta_\theta + \check\Delta_\perp\|_\infty}$ holds $\forall\checkbmx(0) \!\in\! \Omega_0$.
\end{theorem}
\end{mdframed}
\vspace{-0.25cm}
\begin{proof}[Proof Sketch]
We consider the Lyapunov candidate $\vartheta$ of Corollary~\ref{cor:latent_system_converges}.
Its time-derivative now includes the disturbances, taking the form $\dot\vartheta \!=\! -\checkbmx \bm H \checkbmx + \checkbmx^\intercal \begin{psmallmatrix}
    \epsilon \\ \mathbf I
\end{psmallmatrix}\check\Delta$ with $\check\Delta=\check\Delta_\theta + \check\Delta_\perp$.
Using Corollary~\ref{cor:latent_system_converges}, we can rewrite the Lyapunov derivative and show that $\dot\vartheta \!\leq\! -2(1-\xi)\lambda v_2 \|\checkbmx\|^2$ holds $\forall \|\checkbmx\| \!\geq\! \frac{\eta}{2\xi\lambda v_2}\|\check\Delta\|$ with $\xi\!\in\!(0,1)$ and $\eta\!=\!\sqrt{1+\epsilon^2}$.
From Corollary~\ref{cor:bounded_delta_check}, 
we deduce that $\vartheta$ decreases exponentially outside a ball centered at the origin and the error remains within the forward-invariant level set $\Omega_0$.
By the comparison lemma, the trajectory converges exponentially to a tube, $\|\checkbmx(t)\|\leq ce^{-(1-\xi)\lambda t}\|\checkbmx(0)\|+\gamma \|\check\Delta\|_\infty$ with $\gamma = \frac{c\eta}{4\xi\lambda v_2}\sqrt{\frac{\xi}{1-\xi}}$, see App.~\ref{appendix:latent_iss_proof} for details.
\end{proof}

\vspace{-0.1cm}
\subsection{Stability and convergence of the Original System}
\vspace{-0.05cm}
\label{sec:method_part1_fullorder_plant_convergence_analysis}
Theorem~\ref{thm:latentISS} guarantees the bounded convergence of the reduced error state $\check{\bm x}$ under the proposed control law.
However, it does not immediately imply stability and convergence of the original high-dimensional system, for which a learned nonlinear projection with $\ell_{\text{AE}} \!\neq\! 0$ induces additional errors on top of dynamic modeling and projection errors.
Next, we examine the closed-loop dynamics and their properties on the learned embedded submanifold $\varphiq(\checkmathcalq)$ before turning our attention to the original system evolving on $\mathcalq$.\looseness-1

\subsubsection{Stability and convergence analysis on the embedded submanifold $\varphiq(\checkmathcalq)\!\subseteq\! \mathcal Q$}
The nominal system on the embedded submanifold follows the projected Lagrangian dynamics
\begin{equation}
    \label{eq:projectedequationsofmotion}
    \tilde{\bm M}(\tildebmq) \ddtildebmq_{\text{nom}} + (\tilde{\bm C}(\tildebmq, \dtildebmq) + \tilde{\bm D}(\tildebmq)) \dtildebmq + \tilde{\bm g} = \tilde{\bmtau}_{\text{c}},
\end{equation}
with $\tilde{\bmM} \!=\! \bm{P}^\trsp \bmM \bm{P}$, $\tilde{\bm g} \!=\! \bm{P}\bm g$, and $\tilde{\bm D} \!=\! \bm{P}^\trsp \bm D \bm{P}$. 
The nominal projected dynamic equations are exactly identified only if the loss~\eqref{eq:loss_RO_LNN} is $0$ $\forall (\bmq,\dbmq)\!\in\!\mathcal{T}\mathcalq$.
The closed-loop dynamics on $\varphi(\mathcal{T}\checkmathcalq)$ are characterized from the difference between the nominal and learned projected dynamics as follows.
\vspace{-0.2cm}
\begin{mdframed}[hidealllines=true,backgroundcolor=lightlightgray,innerleftmargin=.1cm,
  innerrightmargin=.1cm,innertopmargin=-.1cm,innerbottommargin=0.1cm]
\begin{proposition}
\label{prop:tilde_closed_loop}
On the submanifold $\varphi(\mathcal{T}\checkmathcalq)$, the closed loop dynamics induced by the control torque $\tildebmtau_{\text{c}}$~\eqref{eq:torque_reconstruction} take the form
\begin{equation}
    \label{eq:tilde_closed_loop_dyns}
    \tilde{\bmM}\ddtildebme + (\tilde{\bm C} + \tilde{\bm D} + \tildeKd) \dtildebme + \tildeKp \tildebme = \tilde{\Delta}_{\theta} + \tilde{\Delta}_{\perp},
\end{equation}
where $\tilde{\Delta}_{\theta}$ and $\tilde\Delta_\perp$ are disturbances due to dynamic modeling and projection alignment errors, that are bounded in a compact neighborhood $\mathcal{N}\!\subseteq\!\mathcal{T}\mathcalq$ of the training data distribution.
\end{proposition}
\end{mdframed}
\vspace{-0.2cm}
\begin{proof}[Proof Sketch]
    The closed-loop dynamics on the embedded submanifold  quantify the difference between the nominal submanifold acceleration $\ddtildebmq_{\text{nom}}$ induced from the control torque $\tildebmtau_{\text{c}}$ in the projected dynamics~\eqref{eq:projectedequationsofmotion}, and the acceleration $\ddtildebmq$ obtained by projecting the acceleration response of the full system $\ddbmq$ to $\tildebmtau_{\text{c}}$~\eqref{eq:plant_actuated_via_reconstructed_torque}, i.e., $\ddtildebmq \!=\! d\varphiq\ddcheckbmq +d^2\varphiq \dcheckbmq^2$ with ${\ddcheckbmq \!=\! d\rhoq\vert_{\bmq}\ddbmq +d^2\rhoq\vert_{\bmq} \dbmq^2}$, following a similar reasoning as Proposition~\ref{prop:latent-closedloop}. The complete derivation is provided in App.~\ref{appendix:embedded_closedloop_dynamics}. 

    Similarly as in the latent space, a dynamic modeling disturbance $\tilde\Delta_{\theta}$ arises from the error between the nominal and learned projected dynamic parameters.  
    Moreover, a projection alignment disturbance $\tilde\Delta_\perp$ 
    captures the error between the nominal $\bmM$-orthogonal and learned projections  
    with $\tilde\Delta_\perp\!=\!0$ when the \ac{ae} learns a Riemannian submersion.
    The disturbances are bounded similarly as in Prop.~\ref{prop:latent-closedloop}, see App.~\ref{appendix:embedded_convergence_bounds}.
\end{proof}

\vspace{-0.2cm}
\begin{mdframed}[hidealllines=true,backgroundcolor=lightlightgray,innerleftmargin=.1cm,
  innerrightmargin=.1cm,innertopmargin=-.1cm,innerbottommargin=0.1cm]
\begin{remark}
\label{remark:decoder-curv}
    In constrast to~\eqref{eq:latent_closed_loop_dyns}, the disturbances $\tilde\Delta_{\theta}$ and $\tilde\Delta_{\perp}$ in~\eqref{eq:tilde_closed_loop_dyns} also depend on the decoder curvature via $d^2\varphiq$.
\end{remark}
\end{mdframed}
\vspace{-0.5cm}
\begin{mdframed}[hidealllines=true,backgroundcolor=lightlightgray,innerleftmargin=.1cm,
  innerrightmargin=.1cm,innertopmargin=-.1cm,innerbottommargin=0.1cm]
\begin{remark}
    The loss~\eqref{eq:loss_RO_LNN} implicitly minimizes the dynamic modeling and projection alignment disturbances $\tilde\Delta_\theta$ and $\tilde\Delta_\perp$ as low losses $\ell_{\text{LNN}, d}$ and $\ell_{\text{LNN}, N}$ are only attainable when the dynamic modeling error is small and $\bm M$-orthogonality holds. 
\end{remark}
\end{mdframed}
\vspace{-0.2cm}

The next theorem provides our main result on input-to-state stability of the projected error state $\tilde{\bm x} \!=\! (\tildebme^\intercal, \dtildebme^\intercal)^\intercal$.\looseness-1 

\vspace{-0.15cm}
\begin{mdframed}[hidealllines=true,backgroundcolor=lightlightgray,innerleftmargin=.1cm,
  innerrightmargin=.1cm,innertopmargin=-.1cm,innerbottommargin=0.1cm]
\begin{theorem}[Stability and convergence on embedded submanifold]
    \label{thm:projectedISS}
    The closed-loop dynamics~\eqref{eq:tilde_closed_loop_dyns} on $\varphi(\mathcal T \checkmathcalq)$ under disturbances $\tilde\Delta_\theta$ and $\tilde\Delta_\perp$ are locally exponentially \ac{iss}, i.e., 
    there exists a compact neighborhood $\Omega_0\!\subset\!\Omega_{\mathcal N}$ of the origin and constants $c\!>\!0$, $\lambda\!>\!0$, $\gamma\!>\!0$ such that 
    ${\|\tildebmx(t)\|\leq ce^{-\lambda t}\|\tildebmx(t)\| + \gamma\|\tilde\Delta_\theta + \tilde\Delta_\perp\|_\infty}$ holds $\forall\tildebmx(0) \!\in\! \Omega_0$.
\end{theorem}
\end{mdframed}
\vspace{-0.2cm}
\begin{proof}[Proof]
    Similar to Thm.~\ref{thm:latentISS}, see App.~\ref{appendix:embedded_iss_proof}.
\end{proof}

\subsubsection{Stability and convergence analysis on $\mathcalq$}
Strictly speaking, the physical system only evolves on the embedded submanifold $\varphi(\mathcal{T}\checkmathcalq)$ if it is perfectly reducible, i.e. $S\subseteq\varphi(\mathcal T\check{\mathcal Q})$ for $S$ in~\eqref{eq:set_of_all_solutions} and some embedding $\varphi$, and if the \ac{rolnn} achieves a zero reconstruction loss $\ell_{\text{AE}}\!=\!0$ on $\mathcal{N}\subseteq\mathcal{T}\mathcalq$. 
In most cases, these conditions only approximately hold and the closed-loop dynamics~\eqref{eq:plant_actuated_via_reconstructed_torque} evolves on (a subset of) the full manifold $\mathcal{T}\mathcalq$ rather than on $\varphi(\mathcal{T}\checkmathcalq)$. Therefore, there exists a $(n\!-\!d)$-dimensional nullspace whose dynamics are overlooked by the \ac{rom}. 
As a consequence, using the \ac{rom} to derive control torques introduces a virtual form of underactuation with respect to the true actuation of the system. 

Intuitively, nullspace effects arise as unmodeled disturbances, as the reduced-order dynamic parameters fail to capture the curvature and forces of the nullspace dimensions of the unknown nominal ones $\{\bmM,\bm g,\bm D\}$. 
Therefore, the system may generate nullspace forces $\bmtau_{N}$ in the dual kernel $\mathcal{V}_{\bmq}^*$ of the learned projection, which act as an additional nullspace disturbance $\Delta_N$. 
Setting the reference trajectory entirely in the learned submanifold does not alleviate this issue, as the nullspace dimensions of $\{\bmM(\tildebmq),\bm g(\tildebmq),\bm D(\tildebmq)\}$ may still generate disturbances.

Nullspace effects are exacerbated if the learned projection $\bm{P}$ if not $\bmM$-orthogononal, as it induces cross couplings between the projected and nullspace generalized forces components leading to non-conservative terms.

However, under the standard \ac{mor} assumption that the set $S$ of trajectories~\eqref{eq:set_of_all_solutions} is approximately a subset of $\varphi(\mathcal{T}\checkmathcalq)$ for some $\varphi$, the nullspace dynamics are bounded and implicitly minimized by the loss~\eqref{eq:loss_RO_LNN}. Intuitively, this aligns with the fact that, for many systems, e.g., soft robots or deformable objects, the nullspace dynamics self-stabilize or are physically limited by constraints or operating conditions~\cite{santina19taskspace}.
Next, we conjecture on the stability and convergence of original closed-loop dynamics~\eqref{eq:plant_actuated_via_reconstructed_torque} induced by the proposed control law~\eqref{eq:torque_reconstruction}.

\vspace{-0.1cm}
\begin{mdframed}[hidealllines=true,backgroundcolor=lightlightgray,innerleftmargin=.1cm,
  innerrightmargin=.1cm,innertopmargin=-.1cm,innerbottommargin=0.1cm]
\begin{conjecture}[Extension of Thm.~\ref{thm:projectedISS} to $\mathcalq$]
\label{conj:stability_fullsys}
Under standard \ac{mor} assumptions, the closed-loop dynamics~\eqref{prop:tilde_closed_loop} on the submanifold $\varphi(\mathcal T \checkmathcalq)$ extends to the manifold $\mathcal T \mathcalq$ with two additional disturbances $\Delta_N$ and $\Delta_{\perp,N}$ due to unmodeled nullspace dynamics and cross-couplings, that are bounded in a compact neighborhood $\mathcal{N}\!\subseteq\!\mathcal{T}\mathcalq$ of the training data distribution. From Theorem~\ref{thm:projectedISS}, we conjecture that the resulting closed-loop dynamics under disturbances ${\Delta \!=\! \tilde\Delta_\theta+\tilde\Delta_\perp + \Delta_N + \Delta_{\perp,N}}$ are locally exponentially ISS.
\end{conjecture}
\end{mdframed}
\vspace{-0.2cm}

Finally, we propose the following hybrid \ac{rolnn}-based control law for fully-actuated systems 
\begin{equation}
        \label{eq:fom_pd_lat_ff_controllaw}
    \bmtau_{\text{c}} = d\rhoq\vert_{\bmq}^\intercal\check{\bm \tau}_{\text{FF}} + \bmtau_\text{\ac{pd}},
\end{equation}
that compensates for small nullspace errors by composing the \ac{rolnn} feedforward torques~\eqref{eq:reduced_pd_ff_controllaw} with a full-order \ac{pd} feedback torque $\bmtau_\text{\ac{pd}}=- \Kp\bm e - \Kd\dot{\bm e}$ as in~\eqref{eq:pd_ff_controllaw}.

\section{Reduced-Order Controllers for Indirectly-actuated and Underactuated Systems}
\label{sec:method_part2}
The previous section analyzed fully-actuated systems with $\bm\tau_{\text{c}} \!=\! \bm u$. We now relax this assumption and consider indirectly-actuated and underactuated systems for which $\bm \tau_{\text{c}} \!=\! \bm B(\bmq)\bm u$ with $\bm{u}\in\mathbb{R}^m$ and $m\!\leq\!n$.
In such cases, the control input $\bm{u}$ cannot necessarily accelerate the system in arbitrary directions in $\mathcalq$. For underactuated systems, $\bm \tau_{\text{c}}$ lies on a subbundle of $\mathcal{T}^*\mathcalq$ and the system generally cannot be commanded to track arbitrary reference trajectories. Here, we outline directions for addressing these more complex actuation cases.

First, we extend the \ac{rolnn} architecture to indirectly- and underactuated Lagrangian systems. 
We learn a parametric input matrix $\bm B_\theta \!:\! \mathcal Q \!\to\! \mathbb{R}^{n \times m}$ via a \ac{mlp} leading to control torques $\bm\tau_{\text{c}} \!=\! \bm B_{\theta}(\bm q)\bm u$ and reduced by~\eqref{eq:pullback_dynamic_parameters} to obtain the reduced dynamics. The \ac{mlp} parameters are optimized jointly with the other \ac{rolnn} parameters. 
Second, we propose to control indirectly-actuated and underactuated systems by lifting the reduced control torques~\eqref{eq:reduced_pd_ff_controllaw} to account for the learned actuation pattern as
\begin{equation}
\label{eq:underactuated_torque_reconstruction}
    \tilde{\bm u} = \bm B(\bmq)^{+} d\rhoq\vert_{\bmq}^\intercal \checkbmtau_{\text{c}},
\end{equation}
with $\bm B(\bmq)^{+}$ denoting the identity-weighted generalized inverse of $\bm{B}(\bmq)$.
Finally, we set $\bm u\!=\!\tilde{\bm u}$.

We briefly analyze the properties of the control law~\eqref{eq:underactuated_torque_reconstruction} under the two scenarios $m\!=\!d$ and $m\!<\!d$. When $m\!=\!d$, the latent space dimension matches the actuation dimension. In this case, a \ac{rolnn} achieving a low loss~\eqref{eq:loss_RO_LNN} learns a latent space corresponding to the actuated dynamics along with the actuation patterns. Under the assumption of an accurately-learned actuation matrix $\bm{B}$, the controller achieves similar performances as described in Sec.~\ref{sec:method_part1}, where the virtual form of actuation introduced by the \ac{rom} is similar to that of a truly underactuated system. The case $m\!<\!d$ is more complex as the actuators may not necessarily be able to execute the control torques computed from the \ac{rom}. However, we postulate that the actuation patterns should be extracted by a well-trained \ac{rolnn}, therefore classical control results apply. 

\section{Experimental Validation}
\label{sec:experiments}
We experimentally validate the proposed \ac{rolnn}-based control law and its properties on a simulated pendulum and a real-world plush puppet manipulated by a humanoid robot. Controller gain design is discussed in App~\ref{appendix:gains}. Details about datasets, network architectures, and model training are in App.~\ref{appendix:datasets-models}. Additional results are provided in App.~\ref{appendix:results}.

\begin{figure}
    \centering
    \adjustbox{trim=0.6cm 0.3cm 0.5cm 0.2cm}{
            \includesvg[width=.95\linewidth]{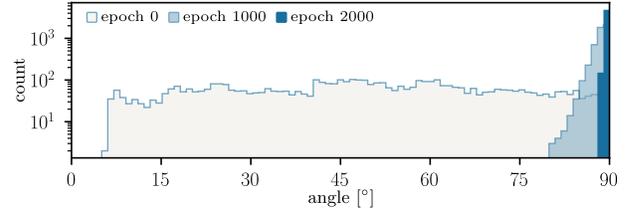}  
        }
    \vspace{-0.1cm}
    \caption{Riemannian angle between the image and kernel of the \ac{ae} projection $\bm P$ at different training epochs, histogram over $5000$ testing points.
    }
    \label{fig:bd_pend:projection_vs_ns_angle_histogram}
    \vspace{-0.65cm}
\end{figure}

\subsection{Simulated Augmented Pendulum}
\label{sec:experiments:pendulum}
We validate our theoretical analysis on the $15$-\ac{dof} augmented pendulum from~\cite{friedl2025hamiltonian}, for which we have access to the ground truth dynamics.
The system is fully actuated and its nonlinear dynamics are specified from the diffeomorphism of a latent $3$-\ac{dof} pendulum augmented with a $12$-\ac{dof} mass-spring mesh. As the mesh oscillations are small, i.e., comparable with reconstruction and dynamic prediction errors in other \ac{rolnn} experiments, the system dynamics are approximately reducible to $3$ dimensions. 
We train a \ac{rolnn} with latent dimension $d\!=\!3$ on $4000$ datapoints along non-actuated and sinusoidal trajectories.\looseness-1 

\textbf{Learned projections.}
Fig.~\ref{fig:bd_pend:projection_vs_ns_angle_histogram} shows the histogram of the Riemannian angle $\alpha$~\eqref{eq:riemannian_angle_between_image_and_nullspace} between projected and nullspace velocity components measured on $1000$ samples from $10$ test trajectories. We observe that $\alpha\!\to\!\sfrac{\pi}{2}$ after training, validating that the \ac{rolnn} learns an approximately $\bmM$-orthogonal projection and minimizes the projection alignment error. Note that a fully $\bmM$-orthogonal projection could be set via a linear embedding $\varphiq(\checkbmq) = \begin{pmatrix} \mathbf I_{3\times 3} \; \mathbf 0_{3\times 12}\end{pmatrix}^\intercal\checkbmq$, leading to a submanifold composed of the $3$ dominant pendulum \acp{dof}. However, this would require knowing the \ac{fom}. Moreover, linear projections are known to display limited expressivity compared to nonlinear subspaces~\cite{Buchfink2024MOR,friedl2025reduced,Otto2023MOR}.

\textbf{Regulation.}
We consider a regulation task in which the controller's goal is to stabilize at a reference configuration. We sample initial and reference configurations uniformly as $q_{0,i}\!\sim\!\mathcal U(-\frac{\pi}{2}, \frac{\pi}{2})$, $q_{\text{d},i}\!\sim\!\mathcal U(-\pi, \pi)$ $\forall i=1,2,3$, and $q_{0,i} \!\sim\! \mathcal U(-0.01, 0.01)$, $q_{\text{d},i} \!\sim\! \mathcal U(-0.015, 0.015)$ $\forall i=4,...,15$.
We compare the proposed \ac{rolnn}-based control law~\eqref{eq:torque_reconstruction} with a \ac{pd} controller on the full system. For completeness, we also consider a linear-intrusive structure-preserving \ac{rom}-based controller following our latent control law~\eqref{eq:torque_reconstruction}. The latent space is fixed via a linear embedding $\varphiq(\checkbmq) = \begin{pmatrix} \mathbf I_{3\times 3} \; \mathbf 0_{3\times 12}\end{pmatrix}^\intercal\checkbmq$ and associated point reduction. As the latent \acp{dof} coincide exactly with the pendulum, we use the analytic dynamic model of the $3$ pendulum \acp{dof}. Note that this linear-intrusive controller is not deployable in practice as it requires the unavailable full knowledge of the \ac{fom} dynamics.

Fig.~\ref{fig:bd_pend:regulation_errors_median_over_time} shows the evolution of the error $\|\bm e\| \!=\! \|\bmq - \bmqdes\|$ over time across $10$ configurations. 
Examples of resulting trajectories are displayed in App.~\ref{appendix:results}. 
We observe that the \ac{rolnn}-based controllers converge more accurately with increased stability compared to the \ac{pd}, which can only achieve regulation up to a steady-state error. Note that the fully-learned \ac{rolnn} partially models the mesh dynamics, resulting in a smoother error decrease compared to the linear-intrusive \ac{rom} that only learns the dynamics of the $3$-\ac{dof} pendulum.

\begin{figure}
    \centering
    \adjustbox{trim=0.5cm 0.3cm 0.5cm 0.2cm}{
            \includesvg[width=.9\linewidth]{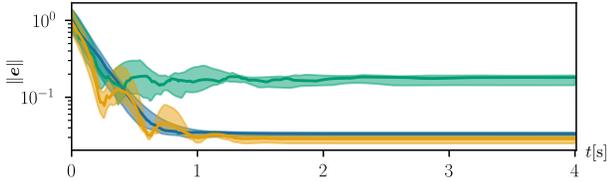}  
        }
    \caption{Regulation error (median and quartiles) for the \ac{rolnn} (\solidblueoneline), linear-intrusive (\solidorangetwoline), and model-free \ac{pd} (\solidgreenthreeline) controllers over $10$ random initial and desired configurations.
    }
    \label{fig:bd_pend:regulation_errors_median_over_time}
    \vspace{-0.1cm}
\end{figure}
\begin{figure}
    \centering
    \adjustbox{trim=0.9cm 0.2cm 0.9cm 0.21cm}{
            \includesvg[width=.9\linewidth]{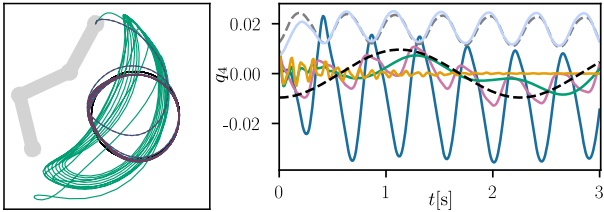}  
        }
    \caption{\emph{Left:} Original (\dashedblackline) and projected (\dashedaubergineline) tracking reference,  and \ac{rolnn}-based (\solidblueoneline), hybrid \ac{rolnn}-based (\solidpinkfourline), and model-free \ac{pd} (\solidgreenthreeline) controllers trajectories. \emph{Right:} Reference and controllers trajectories on a mesh \ac{dof}, also depicting the projected reference (\dashedgreyline) and trajectory (\solidlightblueline) of the \ac{rolnn}-based controller in $\varphi(\mathcal{T}\checkmathcalq)$, and linear-intrusive controller (\solidorangetwoline). 
    }
    \label{fig:bd_pend:ns_dof_tracking}
    \vspace{-0.65cm}
\end{figure}

\textbf{Tracking.}
We evaluate tracking performance on circular task space trajectories with varying center and radius. Note that these references significantly differ from the training data.
Fig.~\ref{fig:bd_pend:ns_dof_tracking}-\emph{left} shows an example of a reference trajectory in task space, along with the projected reference by the \ac{ae} and the trajectories generated by the \ac{rolnn}-based and \ac{pd} controllers. The original and projected reference overlap, showing that the \ac{rolnn} achieves a low reconstruction error. While the \ac{rolnn}-based controller precisely tracks the desired trajectory, the \ac{pd} displays significant tracking errors.
This is validated by Fig.~\ref{fig:bd_pend:tracking_errors_median_over_time}, which displays average tracking errors $\|\bm e\|$ over time. 
Importantly, the tracking error of the \ac{rolnn}-based controllers is due to the untracked nullspace dynamics of the mesh. Fig.~\ref{fig:bd_pend:ns_dof_tracking}-\emph{right} displays the trajectories of the first \acp{dof} of the mesh. As expected, the linear-intrusive \ac{rom}-based controller does not generate any control torques as it entirely disregards the nullspace dynamics. In the case of the fully-learned \ac{rolnn}, the projected reference trajectory is well tracked in $\varphi(\mathcal{T\checkmathcalq})$ by the \ac{rolnn}-based controller. However, it differs from the original reference trajectory due to nullspace dynamics and cross-coupling disturbances, leading to tracking errors on the mesh \acp{dof}. Nevertheless, as the system satisfies the \ac{mor} assumptions, the controller displays a stable and convergent behavior as predicted by Conjecture~\ref{conj:stability_fullsys}.
Finally, Fig.~\ref{fig:bd_pend:tracking_errors_median_over_time} shows that the hybrid \ac{rolnn}-based control law~\eqref{eq:fom_pd_lat_ff_controllaw} achieves the lowest tracking error, showcasing that composing a \ac{rolnn}-based feedforward term with a full-order \ac{pd} term effectively compensates for small nullspace errors.\looseness-1

\begin{figure}
    \centering
    \adjustbox{trim=0.5cm 0.3cm 0.5cm 0.2cm}{
            \includesvg[width=.9\linewidth]{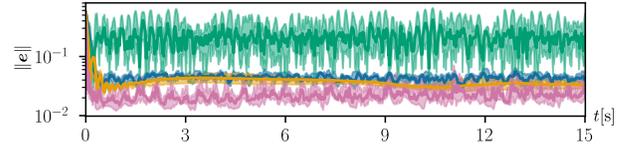}  
        }
    \caption{Tracking error (median and quartile) of the \ac{rolnn} (\solidblueoneline), hybrid \ac{rolnn} (\solidpinkfourline), linear \ac{rolnn} (\solidorangetwoline), and model-free \ac{pd} (\solidgreenthreeline) controllers over $10$ reference trajectories.
    }
    \label{fig:bd_pend:tracking_errors_median_over_time}
    \vspace{-0.4cm}
\end{figure}

\subsection{Regulation on an Underactuated Puppet}
\label{sec:experiments:monkey}
We consider an underactuated scenario where a plush monkey puppet is manipulated by a RB-Y1 humanoid robot, see Fig.~\ref{fig:closedloop}. The puppet state consists of the position and velocity of $351$ keypoints tracked by a custom visual perception pipeline, see App.~\ref{appendix:datasets-models}. The actuation vector $\bm{u}\in\mathbb{R}^3$ is set as the end-effector angular velocity. We train a \ac{rolnn}, including the actuation matrix $\bm{B}$, with latent dimension $d\!=\!6$ on $3000$ datapoints collected with sinusoidal angular velocity commands. 
We test the proposed \ac{rolnn}-based control law on a regulation task, where the robot stabilizes the puppet at a given reference configuration. Fig.~\ref{fig:monkey:frames_regulation} shows resulting trajectories (see also App.~\ref{appendix:results}), showcasing the ability of the \ac{rolnn}-based control law to stably regulate highly-underactuated systems.

\begin{figure}[]
    \centering
    \includegraphics[width=0.9\linewidth]{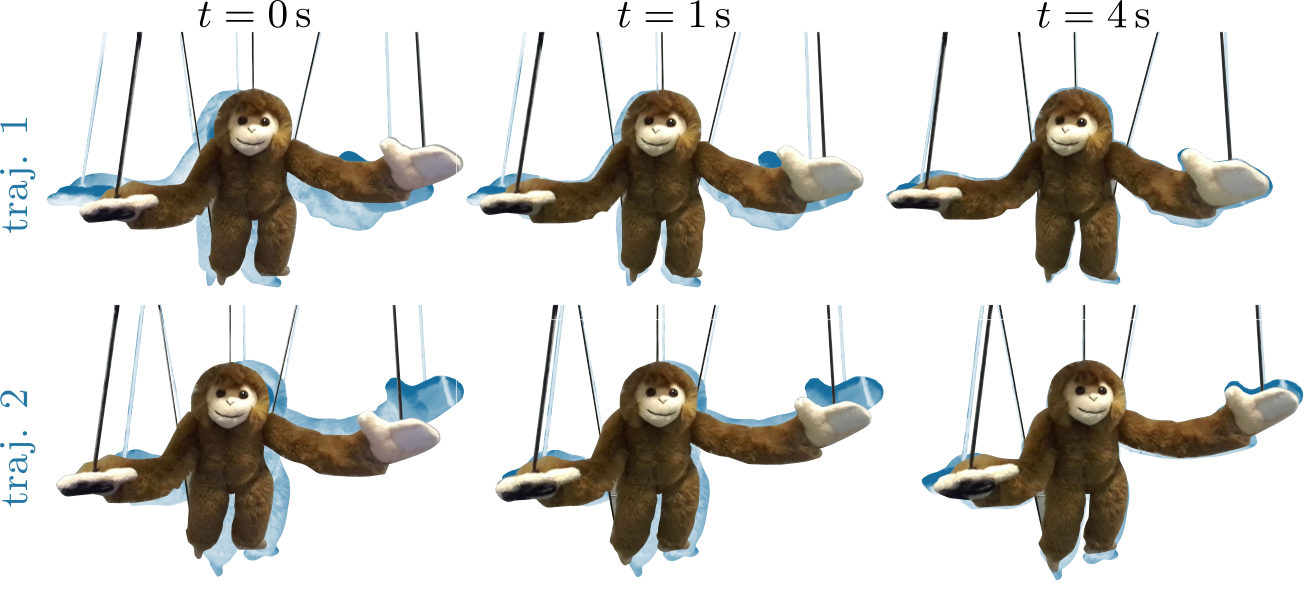}
    \vspace{-0.4cm}
    \caption{Regulated trajectories of a puppet generated by the \ac{rolnn}-based controller for three different reference configurations. 
    }
    \vspace{-0.65cm}
    \label{fig:monkey:frames_regulation}
\end{figure}

\section{Discussion}
\label{sec:conclusion}
This paper presents a novel latent tracking controller for high-dimensional Lagrangian systems with unknown dynamics.
Our control design leverages the Lagrangian latent spaces of the non-intrusive \acp{rolnn}, thereby enabling formal stability and convergence guarantees typically overlooked by latent neural architectures.
Taking a Riemannian perspective, we characterized projection-based errors and discussed the influence of nullspace \acp{dof} in the control design.
Our experimental results validate that learned structure-preserving \acp{rom} lead to effective control of unknown dynamics of high-dimensional robotic systems.
The presented analysis opens avenues towards principled application of structure-preserving \ac{mor} to controlling complex, high-dimensional robotic systems.

Current limitations include the dependence of the presented strategy for real-world dynamic tasks to high-frequency full-state feedback. Integrating \ac{rom}-based observers could increase practical feasibility.
Moreover, real-world tracking performance also depends heavily on the underactuation structure. In future work, we will investigate how Lagrangian latent spaces can help with planning, e.g., by learning to match the actuated subspace.
We will also explore more diverse and efficient control strategies tailored to a learned \ac{rom}, aiming for additional guarantees on robustness and optimality.

\section*{Acknowledgments}
This work was supported by ERC AdV grant BIRD 884807, Swedish Research Council, and the Knut and Alice Wallenberg Foundation, including the Wallenberg AI, Autonomous Systems and Software Program (WASP).
The computations were enabled by the Berzelius resource provided by the Knut and Alice Wallenberg Foundation at the National Supercomputer Centre.

%% Use plainnat to work nicely with natbib. 
\clearpage
\bibliographystyle{plainnat}
\bibliography{references}

\clearpage
\appendix

\subsection{Proof of Proposition~\ref{prop:latent-closedloop} --- Part 1: Derivation of the Latent Closed-loop Dynamics}
\label{appendix:latent_closedloop_dynamics}
The full high-dimensional closed-loop dynamics are given by~\eqref{eq:plant_actuated_via_reconstructed_torque}, with lifted control torque $\tildebmtau_{\text{c}}$ via the tracking law~\eqref{eq:reduced_pd_ff_controllaw} as system input.
Using the second derivative of the point reduction $\rhoq$, the projection of the true closed-loop acceleration $\ddbmq$ into the latent space takes the form
\begin{equation}
    \label{eq:latent_plant_acceleration}
    \ddcheckbmq = d\rhoq\left(\bm M^{-1}(\tildebmtau_{\text{c}} - (\bm C + \bm D)\dbmq - \bm g)\right) + d^2\rhoq \dbmq^2.
\end{equation}

We aim to characterize the difference between the latent acceleration $\ddcheckbmq$~\eqref{eq:latent_plant_acceleration} obtained by encoding the true system acceleration $\ddbmq$ and the nominal latent acceleration $\ddcheckbmq_{\text{nom}}$ induced directly from the latent control torque $\checkbmtau_{\text{c}}$ in~\eqref{eq:reducedequationsofmotion}, as expanded as\looseness-1
\begin{equation*}
    \checkbmM\ddcheckbmq = \checkbmM\ddcheckbmq_{\text{nom}} + \checkbmM(\ddcheckbmq - \ddcheckbmq_{\text{nom}}),
\end{equation*}
in~\eqref{eq:diff_latent_accelerations}.
Substituting $\ddcheckbmq$ from~\eqref{eq:latent_plant_acceleration} and $\ddcheckbmq_{\text{nom}}$ from~\eqref{eq:reducedequationsofmotion} with $\checkbmtau\!=\!\checkbmtau_{\text{c}}$, we rewrite~\eqref{eq:diff_latent_accelerations} as
\begin{equation*}
\begin{aligned}
    \checkbmM\ddcheckbmq = &\big(\checkbmtau_{\text{c}} - (\checkbmC + \checkbmD) \dcheckbmq - \checkbmg \big) \\ &+ \checkbmM \Big( d\rhoq \bmM^{-1} \big(\tildebmtau_{\text{c}} - (\bm{C} + \bm{D}) \dbmq - \bm{g} \big) + d^2\rhoq \dbmq^2 \\ & \quad\quad\quad  - \checkbmM^{-1} \big(\checkbmtau_{\text{c}} - (\checkbmC + \checkbmD) \dcheckbmq - \checkbmg \big) \Big)
\end{aligned}
\end{equation*}
Substituting $\tildebmtau_{\text{c}}$ by~\eqref{eq:torque_reconstruction}, $\checkbmtau_{\text{c}}$ by~\eqref{eq:reduced_pd_ff_controllaw}, and using $\ddcheckbmq\!=\! \ddcheckbme + \ddcheckbmqdes$ and $\dcheckbmq\!=\! \dcheckbme + \dcheckbmqdes$, we obtain
\begin{equation*}
\begin{aligned}
\label{eq:latent_closed_loop_expanded_full}
    \checkbmM(\ddcheckbme+ \ddcheckbmqdes) = 
    & \Big(\check{\bm M}_{\bm\theta}\ddcheckbmqdes + (\check{\bm C}_{\bm\theta} + \check{\bm D}_{\bm\theta})\dcheckbmqdes + \check{\bm g}_{\bm\theta} - \checkKp\check{\bm e} - \checkKd \dot{\check{\bm e}} \\ 
    & \quad\quad - (\checkbmC + \checkbmD) (\dcheckbme + \dcheckbmqdes) - \checkbmg \Big) \\
    &+ \checkbmM d\rhoq \bmM^{-1}\Big(d\rhoq^\trsp \checkbmtau_{\text{c}} - (\bm{C} + \bm{D}) \dbmq - \bm{g}\Big) \\ 
    &+ \checkbmM d^2\rhoq \dbmq^2 \\ 
    & - \Big(\checkbmtau_{\text{c}} - (\checkbmC + \checkbmD) \dcheckbmq - \checkbmg \Big),
\end{aligned}
\end{equation*}
which can then simply be rearranged into the form~\eqref{eq:latent_closed_loop_dyns} with dynamic modeling disturbance
\begin{equation}
\label{eq:check_delta_dyn_def_appendix}
    \check\Delta_\theta = \check{\bm M}_{\bm\theta}\ddcheckbmqdes + (\check{\bm C}_{\bm\theta} + \check{\bm D}_{\bm\theta})\dcheckbmqdes + \check{\bm g}_{\bm\theta} - \checkbmM \ddcheckbmqdes - (\checkbmC + \checkbmD) \dcheckbmqdes - \checkbmg ,
\end{equation}
and projection alignment disturbance
\begin{equation}
\begin{aligned}
\label{eq:check_delta_perp_projection_alignment_error_def_appendix}
    \check\Delta_\perp = &\checkbmM d\rhoq \bmM^{-1}d\rhoq^\trsp \checkbmtau_{\text{c}} - \checkbmtau_{\text{c}} \\
    &+  \checkbmM d^2\rhoq \dbmq^2 \\
    &+ (\checkbmC + \checkbmD) \dcheckbmq -  \checkbmM d\rhoq \bmM^{-1} (\bm{C} + \bm{D}) \dbmq \\
    &+ \checkbmg - \checkbmM d\rhoq \bmM^{-1} \bm g.
    \end{aligned}
\end{equation}
It is straightforward to rearrange~\eqref{eq:check_delta_dyn_def_appendix} into the form~\eqref{eq:delta_check_parameter_convergence_disturbance} and~\eqref{eq:check_delta_perp_projection_alignment_error_def_appendix} into~\eqref{eq:delta_perp_orthogonality_convergence_disturbance}.

\subsection{Proof of Proposition~\ref{prop:latent-closedloop} --- Part 2: Derivation of the Latent Disturbances Bounds}
\label{appendix:latent_convergence_bounds}
Next, we provide bounds on the latent disturbances $\check\Delta_\theta$ and $\check\Delta_\perp$ and discuss the conditions under which the \ac{rolnn} models the latent system exactly.

\subsubsection{Dynamic Modeling Disturbance}
Importantly, the losses $\ell_{\text{LNN},d}$ and $\ell_{\text{LNN},d}$ are not solely influenced by dynamic parameter convergence, but also depend on the projection alignment error. As a consequence, the losses can only equal zero when also the projection alignment error vanishes. Nevertheless, small values of $\ell_{\text{LNN},d}$ and $\ell_{\text{LNN},d}$ indicate a low dynamic modeling error $\check\Delta_\theta$. 

We assume a neighborhood $\mathcal N\subseteq\mathcal T\mathcal Q$ close to the training data distribution on which the original system matrices are well-behaved, i.e., bounded. Moreover, we note that the learned \ac{rolnn} employs smooth, differentiable activation functions in both the autoencoder and latent \ac{lnn}, and that its parameters are bounded after training.
This allows us to conclude that, for any $(\bm q, \dbmq)\in \mathcal N$ and bounded reference trajectories, the dynamic modeling disturbance is bounded, i.e., there exists a constant $\check r_\theta>0$ such that ${\|\check\Delta_\theta (\bmq, \dbmq, \bmqdes, \dbmqdes, \ddbmqdes)\|\leq \check r_\theta}$.
In other words, this boundedness simply follows from the smoothness of the learned dynamics and the compactness of the neighborhood $\mathcal N$. 

\subsubsection{Projection Alignment Disturbance}

Given~\eqref{eq:check_delta_perp_projection_alignment_error_def_appendix}, we first aim to show that $\check\Delta_\perp \!=\! 0$ when the autoencoder learns a Riemannian submersion, i.e., if $d\rhoq\vert_{\varphiq(\bmq)} \!=\! d\varphiq\vert^{+_{,\bm M}}_{\checkbmq}$ holds. Second, we derive boundedness of the error on $\mathcal N\subseteq \mathcal T\mathcal Q$.

To proceed, we first define the latent Coriolis term of the latent nominal (a.k.a. pullback) and projected dynamics.
For the nominal dynamics, the Coriolis force is derived equivalently as in full-order dynamics~\eqref{eq:equationsofmotion}, through the Christoffel symbols of the pullback metric $\checkbmM(\varphi(\checkbmq))$, yielding
\begin{equation}
\label{eq:pullback_coriolis}
    \checkbmC\checkbmq = d\varphiq^\intercal \bm Cd\varphiq \dcheckbmq + d\varphiq^\intercal \bm M d^2\varphiq \dcheckbmq^2.
\end{equation}
The latent Coriolis force of the projected dynamics is projected through the encoder, accounting for curvature in the derivatives, as
\begin{equation}
\label{eq:latent_coriolis_via_proj}
    \checkbmC_{\text{proj}}\checkbmq = \checkbmM d\rhoq\bm M^{-1} \bm C \dbmq-\checkbmM d^2\rhoq \dbmq^2.
\end{equation}
As stated in~\cite[Thm. 5.7]{Buchfink2024MOR}, the latent nominal and projected Coriolis forces coincide for all states on $\varphi(\mathcal T\mathcal Q)$ when $d\rhoq\vert_{\varphiq(\bmq)} \!=\! d\varphiq\vert^{+_{,\bm M}}_{\checkbmq}$ holds. This can be verified by setting $d\rhoq\vert_{\varphiq(\bmq)}$ in~\eqref{eq:latent_coriolis_via_proj} as the generalized inverse~\eqref{eq:riemannian_pseudoinverse_of_embedding} --- which is equivalently written as $d\varphiq\vert^{+_{,\bm M}}_{\checkbmq} = \checkbmM(\checkbmq)^{-1} \; d\varphiq\vert_{\checkbmq}^\intercal \; \bm M(\bm q)$ --- and using the identity 
\begin{equation}
\label{eq:app_simplification_d2_rho}
    d^2\rhoq d\varphiq^2\dcheckbmq^2 = -d\rhoq d^2\varphiq \dcheckbmq^2,
\end{equation} 
which is obtained by differentiating the second projection property in~\eqref{eq:projectionproperty} as $\frac{\partial}{\partial \checkbmq}d\rhoq\vert_{\varphiq(\checkbmq)}d\varphiq\vert_{\checkbmq} = \bm 0$.

We then return to with the projection alignment disturbance in the form~\eqref{eq:check_delta_perp_projection_alignment_error_def_appendix}. Considering states $(\bmq, \dbmq)\in \mathcal T\mathcal Q$, i.e., for which $d\varphiq \dcheckbmq \equiv \dbmq$, plugging in the nominal pullback Coriolis forces $\checkbmC\dcheckbmq$~\eqref{eq:pullback_coriolis}, and using~\eqref{eq:app_simplification_d2_rho}, the expression reads as
\begin{align}
    \label{eq:app_check_delta_perp_simplified}
    \check\Delta_\perp =& \checkbmM d\rhoq \bm M^{-1} d\rhoq^\intercal \checkbmtaudes-\checkbmtaudes \nonumber \\ 
    &+ d\varphiq^\intercal \bm M d^2\varphiq \dcheckbmq^2 - \checkbmM d\rhoq d^2\varphiq \dcheckbmq^2 \nonumber\\ 
    &+ d\varphiq^\intercal \bm C d\varphiq \dcheckbmq - \checkbmM d\rhoq \bm M^{-1}\bm C d\varphiq \dcheckbmq\\ 
    &+\checkbmD \dcheckbmq - \checkbmM d\rhoq \bm M^{-1}\bm D d\varphiq \dcheckbmq\nonumber \\ 
    &+ \checkbmg - \checkbmM d\rhoq\bm M^{-1} \bm g \nonumber.
\end{align}

\textbf{$\check\Delta_\perp$ for a $\bm M$-orthogonal projection.}
In the case where the AE learns a $\bmM$-orthogonal projection, we have ${d\rhoq \!=\! d\varphiq^{+_{,\bm M}}}$. Substituting $d\rhoq$ by the generalized inverse~\eqref{eq:riemannian_pseudoinverse_of_embedding} into the projection alignment disturbance~\eqref{eq:app_check_delta_perp_simplified} yields 
\begin{align*}
    \check\Delta_\perp =& d\varphiq^\intercal d\rhoq^\intercal \check{\bm\tau}_{\text{c}} - \check{\bm\tau}_{\text{c}} \\
    &+ \checkbmC\dcheckbmq - (d\varphiq^\intercal \bm C d\varphiq\dcheckbmq + d\varphiq^\intercal \bm M d^2\varphiq \dcheckbmq^2)\\
    &+ \checkbmD \dcheckbmq - d\varphiq^\intercal \bm D d\varphiq \dcheckbmq \\
    &+ \check {\bm g} - d\varphiq^\intercal \bm g \\
    =& 0.
\end{align*}
The equality to $0$ is obtained using the identity $d\varphiq^\intercal d\rhoq^\intercal =\bm I$ from the transpose of the tangent projection property~\eqref{eq:projectionproperty}, the pullback dissipation~\eqref{eq:pullback_dissipation}, the potential and Coriolis forces from~\eqref{eq:pullback_dynamic_parameters}, and the latent nominal Coriolis force~\eqref{eq:pullback_coriolis}.

\textbf{General bound on $\check\Delta_\perp$.}
Next, we derive an upper bound on $\check\Delta_\perp$ for any other nontrivial case. 
First, we further re-arrange~\eqref{eq:app_check_delta_perp_simplified} using the pullback dynamic parameters~\eqref{eq:pullback_dynamic_parameters},~\eqref{eq:pullback_dissipation} and write, for states $(\bmq, \dbmq) = (\tildebmq, \dtildebmq) = (\varphiq(\checkbmq), d\varphiq \dcheckbmq)$, $\check\Delta_\perp$ as a sum of three terms
\begin{align}
        \label{eq:app_check_delta_perp_sum_of_terms_form}
    \check\Delta_\perp &= \underbrace{(\checkbmM d\rhoq \bm M^{-1}d\rhoq^\intercal - \bm I)\checkbmtaudes}_{\check\Delta_{\perp,1}} \nonumber\\
    &+  \underbrace{d\varphiq^\intercal \bm M (\bm I - d\varphiq d\rhoq)d^2\varphiq \dcheckbmq^2}_{\check\Delta_{\perp,2}} \\
    &+ \underbrace{ (d\varphiq^\intercal - \checkbmM d\rhoq \bm M^{-1})((\bm C + \bm D)\dbmq + \bm g)}_{\check\Delta_{\perp,3}}, \nonumber
\end{align}
equivalently to~\eqref{eq:delta_perp_orthogonality_convergence_disturbance}.
Notice that each term in~\eqref{eq:app_check_delta_perp_sum_of_terms_form} is, on a compact state-space, bounded by definition. Here, we explicitly quantify each bound as a function of the system state and the the angle $\sfrac{\pi}{2}-\alpha$ between the nominal $\bmM$-orthogonal and learned projections $\bm{P}_\perp$ and $\bm{P}$, see Fig.~\ref{fig:submersion_angles}, on which the projection alignment disturbance directly depends. 

To derive the individual upper bounds, we write the error between the learned tangent reduction $d\rhoq$ and the generalized inverse of the differential of the learned embedding $d\varphiq$ as 
\begin{equation}
\label{eq:error_proj_submersion}
     \bm E_\perp = d\rhoq - d\varphiq^{+_{,\bm M}}.
\end{equation}
As $\text{im}(\bm P_\perp) \!=\! \text{im}(\bm P)$ (see Fig.~\ref{fig:submersion_angles}), we have 
\begin{align}
    & \bm E_\perp d\varphiq = \bm 0, \\ 
    & d\varphiq \bm E_\perp = \bm P - \bm P_\perp. \label{eq:error_proj_perp}
\end{align}
Moreover, the error $\bm E_\perp$ is small when the \ac{rolnn} achieves a low loss~\eqref{eq:loss_RO_LNN}, as a low losses $\ell_{\text{LNN},d}$ and $\ell_{\text{LNN},N}$ are only attainable when both the dynamic parameters are well learned and the \ac{ae}-orthogonality hold.

We can now rewrite the three terms of~\eqref{eq:app_check_delta_perp_sum_of_terms_form} as functions of the nominal $\bmM$-orthogonal and learned projections $\bm{P}_\perp$ and $\bm{P}$, and error~\eqref{eq:error_proj_submersion}. 
By substituting $\bm I \!=\! \checkbmM d\varphiq^{+_{,\bm M}}\bm M^{-1} {d\varphiq^{+_{,\bm M}}}^\intercal$ in $\check\Delta_{\perp,1}$, we obtain 
\begin{equation*}
\check{\Delta}_{\perp,1} = \checkbmM \epsilon_\perp\bm M^{-1} \bm E_\perp^\intercal \checkbmtaudes = d\varphiq^\intercal \bm M(\bm P- \bm P_\perp)\bm M^{-1} \bm E_\perp^\intercal \checkbmtaudes.
\end{equation*}
Expanding $d\varphiq^\intercal \bm M\!=\!d\varphiq^\intercal \bm M d\varphiq d\varphiq\vert^{+_{,\bm M}} $ in $\check\Delta_{\perp,2}$ yields 
\begin{equation*}
\check{\Delta}_{\perp,2} = -d\varphiq^\intercal \bm M(\bm P - \bm P_\perp)d^2\varphiq\dcheckbmq^2.
\end{equation*}
With the same expansion as for $\check\Delta_{\perp,2}$, we get 

\small
\vspace{-0.4cm}
\begin{equation*}
\begin{aligned}
    \check{\Delta}_{\perp,3} &= d\varphiq^\intercal \bm M (d\varphiq d\varphiq\vert^{+_{,\bm M}} - d\varphiq d\rhoq)\bm M^{-1}(\bm g + \bm C\dbmq + \bm D \dbmq) \\ 
    &= -d\varphiq^\intercal \bm M(\bm P - \bm P_\perp)\bm M^{-1}(\bm g + \bm C\dbmq + \bm D \dbmq).
    \end{aligned}
\end{equation*}
\normalsize

From~\eqref{eq:error_proj_perp}, we obtain the projection-difference-based upper bound $\|\bm E_\perp^\intercal\|_{\bm M}\leq |\cot \alpha| \;\|d\varphiq^{+_{,\bm M}}\|_{\bm M}$, where we used the geometric intuition of Fig.~\ref{fig:submersion_angles} and the norm equivalence 
\begin{equation*}
\|\bm X\|_{\bm M} = \|\bm M^{\frac{1}{2}}\bm X \bm M^{-\frac{1}{2}}\| = \sqrt{\frac{\lambda_{\max}(\bm M)}{\lambda_{\min}(\bm M)}}\|\bm X\|.
\end{equation*}
Together with Cauchy-Schwarz for the products and the norm equivalence, we finally obtain the following individual bounds

\small
\vspace{-0.4cm}
\begin{align*}
    &\|\check{\Delta}_{\perp,1}\| \leq |\cot \alpha|^2 \; \frac{\lambda_{\max}(\bm M)}{\lambda_{\min}(\bm M)} \;\|d\varphiq\| \; \|d\varphi^{+}\| \; \|\checkbmtaudes\|, \\
    &\|\check{\Delta}_{\perp,2}\| \leq |\cot \alpha| \; \lambda_{\max}(\bm M) \; \|d\varphiq\| \; \|d^2\varphiq\| \; \|\dcheckbmq\|^2, \\
    & \|\check{\Delta}_{\perp,3}\| \leq |\cot \alpha| \sqrt{\frac{\lambda_{\max}(\bm M)}{\lambda_{\min}(\bm M)} } \|d\varphiq\|  (\|\bm g\| + \|\bm C\|\|\dbmq\| +  \|\bm D\|\|\dbmq\|),
\end{align*}
\normalsize
where $\checkbmtaudes$ depends current state, encoded references, and \ac{rolnn} dynamic parameters.

This allows us to conclude that, for any $\forall (\bm q, \dbmq)\in \mathcal N$ and bounded reference trajectories, the projection alignment disturbance is small and bounded, i.e., there exists a constant $\check r_\perp>0$ such that $\|\check\Delta_\perp (\bmq, \dbmq, \bmqdes, \dbmqdes, \ddbmqdes)\| \leq \|\check{\Delta}_{\perp,1}\| + \|\check{\Delta}_{\perp,2}\| + \|\check{\Delta}_{\perp,3}\| \leq \check r_\perp$.

\subsection{Proof of Corollary~\ref{cor:latent_system_converges} --- Latent Nominal Closed-loop Stability}
\label{appendix:latent_closed_loop_errorfree_proof}

We analyze the stability of the nominal latent closed-loop dynamics, i.e.,~\eqref{eq:latent_closed_loop_dyns} with $\check\Delta_\theta=\check\Delta_\perp=0$, for $\check{\bm x} = (\checkbme^\intercal, \dcheckbme^\intercal )^\intercal$.
For a given bounded reference trajectory, $\Omega_{\mathcal N}$ denotes the neighborhood in the latent error state space where $\checkbmx$ remains small enough such that $(\bm q, \dbmq)\in \mathcal N$, where $\mathcal N\subseteq\mathcal{T}\mathcalq$ is a neighborhood close to the training distribution on which all system matrices are well-behaved.

We consider the Lyapunov candidate 
\begin{equation}
\label{eq:lyapunov_function}
    \vartheta(\checkbmx)=\frac{1}{2} \checkbmx^\intercal \bm \Theta\checkbmx \enspace \text{with} \enspace \bm \Theta = \begin{pmatrix}\checkKp & \epsilon \checkbmM\\ \epsilon \checkbmM  &\checkbmM \end{pmatrix},
\end{equation} 
with scaling constant $\epsilon>0$. 
On a compact state-space, the positive definiteness conditions on the Schur complement imply that $\bm \Theta\!\in\!\SPD$ for sufficiently small $\epsilon^2\!<\!\frac{\lambda_{\min}(\checkKp)}{\lambda_{\max}(\check{\bm M})}$, with $\lambda_{\min}$ and $\lambda_{\max}$ denoting the minimum and maximum eigenvalues of the given matrix.
Then, $v_1\|\check{\bm x}\|^2 \!\leq\! \vartheta(\check{\bm x}) \!\leq\! v_2 \|\check{\bm x}\|^2$ holds for positive constants $(v_1, v_2) = (\lambda_{\min}(\bm \Theta), \lambda_{\max}(\bm \Theta))$.

Using $\dot{\check{\bm M}}=2\checkbmC$, the time-derivative of $\vartheta$ is given by
\small
\vspace{-0.1cm}
\begin{equation}
\label{eq:lyapunov_derivative_simple}
    \dot\vartheta = - \checkbmx^\intercal \bm H \checkbmx\enspace\text{with}\enspace \bm H\!=\!\begin{psmallmatrix}
\epsilon\checkKp & 0.5\epsilon(\checkKd + \checkbmD - \checkbmC) \\ 0.5\epsilon(\checkKd + \checkbmD - \checkbmC)^\intercal & (\checkKd + \checkbmD - \epsilon\check{\bm M})
\end{psmallmatrix}.
\end{equation}
\normalsize
Using the positive definiteness conditions on the Schur complement, there exists a sufficiently small $\epsilon$, for which $\bm H\in \SPD$ on $\Omega_{\mathcal N}$, i.e., the gain matrices dominate nonlinear dynamic effects. Therefore, there exists a level set ${\Omega_0 = \left\{\checkbmx\:|\:\vartheta(\checkbmx) \leq c_0\right\}\subset \Omega_{\mathcal N}}$ such that, for any $\checkbmx(0)\in\Omega_0$, $\dot{\vartheta}<0$ ensures that the trajectory remains in $\Omega_0$ for all $t\geq 0$.

Similarly as for $\vartheta$, $v'_1\|\check{\bm x}\|^2 \!\leq\! -\dot\vartheta(\check{\bm x}) \!\leq\! v'_2 \|\check{\bm x}\|^2$ holds for positive constants $(v'_1, v'_2) = (\lambda_{\min}(\bm H), \lambda_{\max}(\bm H))$.
Using $\dot{\vartheta}\leq - v'_1\|\checkbmx\|^2$ and $\vartheta\leq v_2 \|\check{\bm x}\|^2$, we obtain 
\begin{equation}
\label{eq:dthetabound_nodisturbance}
    \dot{\vartheta}\leq - 2\lambda\vartheta \leq -2 \lambda v_2 \|\check{\bm x}\|^2\quad\text{ with }\quad \lambda = \frac{1}{2} \frac{v'_1}{v_2}.
\end{equation}
From the comparison lemma, it follows that $\vartheta(t) \!\leq\! \vartheta(0)e^{-2\lambda t}$. 
Utilizing the lower bound $v_1\|\checkbmx\|^2\leq \vartheta$ to map back to states, we obtain 
\begin{equation}
\|\checkbmx(t)\| \leq c e^{-\lambda t}\|\checkbmx(0)\| \quad \text{with}\quad c\!=\!\sqrt{\frac{v_2}{v_1}},
\end{equation}
thereby proving that the nominal latent closed-loop dynamics are locally exponentially stable.

\subsection{Proof of Theorem~\ref{thm:latentISS} --- Latent Stability and Convergence}
\label{appendix:latent_iss_proof}
We now extend the stability analysis of the latent closed-loop dynamics~\eqref{eq:latent_closed_loop_dyns} under nonzero dynamic modeling and projection alignment disturbances. 
Again, we consider the Lyapunov candidate~\ref{eq:lyapunov_function}, for which, for sufficiently small $\epsilon>0$, $v_1\|\check{\bm x}\|^2 \!\leq\! \vartheta(\check{\bm x}) \!\leq\! v_2 \|\check{\bm x}\|^2$ holds for positive constants $(v_1, v_2) = (\lambda_{\min}(\bm \Theta), \lambda_{\max}(\bm \Theta))$.

The time derivative of $\vartheta$ now includes the disturbances, taking the form
\begin{equation}
    \label{eq:lyapunov_derivative_simple_disturbed}
    \dot\vartheta = - \checkbmx^\intercal \bm H \checkbmx + \checkbmx^\intercal\begin{pmatrix}
        \epsilon\\ \bm I
    \end{pmatrix} \check\Delta,
\end{equation}
with $\check\Delta=\check\Delta_\theta + \check\Delta_\perp$.
The disturbance term, i.e., second summand, in the time derivative~\eqref{eq:lyapunov_derivative_simple_disturbed} is upper-bounded as 
\begin{equation}
\label{eq:bound_disturbanceterm}
\checkbmx^\intercal\begin{pmatrix}
        \epsilon\\ \bm I
    \end{pmatrix} \check\Delta \leq\eta\|\checkbmx\|\|\check\Delta\| \quad\text{ with }\quad \eta = \sqrt{1+\epsilon^2}.
\end{equation}
Using~\eqref{eq:dthetabound_nodisturbance} from Corollary~\ref{cor:latent_system_converges} and~\eqref{eq:bound_disturbanceterm}, we obtain an upper bound for $\dot\vartheta$ as
\begin{equation}
    \dot\vartheta \leq -2\lambda v_2\|\checkbmx\|^2 + \eta\|\checkbmx\|\|\check\Delta\|
\end{equation}
which can be expanded with $\xi\!\in\!(0,1)$ as
\begin{equation}
\label{eq:dtheta_seminegativedef}
    \dot\vartheta \leq -2(1-\xi)\lambda v_2\|\checkbmx\|^2 -2\xi\lambda v_2\|\checkbmx\|^2 + \eta\|\checkbmx\|\|\check\Delta\|.
\end{equation}
From~\eqref{eq:dtheta_seminegativedef}, $\dot\vartheta$ is strictly negative-definite and the error state $\checkbmx$ decays exponentially at rate $-2(1-\xi)\lambda v_2$ only when ${2\xi\lambda v_2 \|\checkbmx\|^2 \!\geq\! \eta\|\checkbmx\|\|\check\Delta\|}$. 
Therefore, the exponential decay holds for $\|\checkbmx\|\geq \frac{\eta}{2\xi\lambda v_2}\|\check\Delta\|$.

Intuitively, the system is locally exponentially \ac{iss} if every trajectory starting in the level set ${\Omega_0 = \left\{\checkbmx\:|\:\vartheta(\checkbmx) \leq c_0\right\}\subset \Omega_{\mathcal N}}$ remains in $\Omega_0$ for all $t\geq 0$. This implies that the disturbance must be bounded such that the level set ${\Omega_d = \left\{\checkbmx\:|\:\vartheta(\checkbmx) \leq c_d \text{ and } \checkbmx\leq  \frac{\eta}{2\xi\lambda v_2}\|\check\Delta\|\right\}\subset \Omega_{0}}$, corresponding to the states for which the exponential decay does not hold, 
is fully contained inside $\Omega_0$. Note that, while the exponential decay holds in $\Omega_0/\Omega_d$, we cannot guarantee further decrease of energy in $\Omega_d$. 
From Corollary~\ref{cor:bounded_delta_check} the disturbance is bounded as $\sup_{t\geq 0}\|\check\Delta\|=\bar d_{\check\Delta}$, thus, if $\bar d_{\check\Delta}$ is small enough, ensuring that $\Omega_d \subset\Omega_0$.
Therefore, $\dot\vartheta<0$ $\forall \checkbmx \in\Omega_0$, $\Omega_0$ is forward invariant, and the system is locally exponentially \ac{iss}.
Then, by the comparison lemma, we have 
\begin{equation*}
    \vartheta \leq \vartheta(0)e^{-2(1-\xi)\lambda t} + \frac{\eta^2}{8\xi(1-\xi)\lambda^2 v_2}\|\check\Delta\|^2_\infty, 
\end{equation*}
and the trajectory converges exponentially as
\begin{equation*}
\|\checkbmx(t)\|\leq ce^{-(1-\xi)\lambda t}\|\checkbmx(0)\|+\gamma \|\check\Delta\|_\infty,
\end{equation*}
with $\gamma = \frac{c\eta}{4\xi\lambda v_2} \sqrt{\frac{\xi}{1-\xi}}$.

\subsection{Proof of Proposition~\ref{prop:tilde_closed_loop} --- Part 1: Derivation of the Closed-loop Dynamics on the Embedded Submanifold}
\label{appendix:embedded_closedloop_dynamics}
The full high-dimensional closed-loop dynamics are given by~\eqref{eq:plant_actuated_via_reconstructed_torque}, with lifted control torque $\tildebmtau_{\text{c}}$ via the tracking law~\eqref{eq:reduced_pd_ff_controllaw} as system input.
Using the second derivative of the decoder, the projection of the true closed-loop acceleration $\ddbmq$ into the submanifold is $\ddtildebmq \!=\! d\varphiq\ddcheckbmq +d^2\varphiq \dcheckbmq^2$. Incorporating the latent acceleration $\ddcheckbmq$~\eqref{eq:latent_plant_acceleration}, we obtain
\begin{equation}
    \label{eq:tilde_projected_plant_acceleration}
    \ddtildebmq = d\varphiq d\rhoq \bm M^{-1}(\tildebmtaudes -(\bm C+\bm D)\dbmq -\bm g) + d\varphiq d^2\rhoq\dbmq^2 + d^2\varphiq \dcheckbmq^2.
\end{equation}

Following a similar reasoning as in the proof of Proposition~\ref{prop:latent-closedloop} developed in App.~\ref{appendix:latent_closedloop_dynamics}, 
we aim to characterize the difference between the projected acceleration $\ddtildebmq$~\eqref{eq:tilde_projected_plant_acceleration} obtained by projecting the true system acceleration $\ddbmq$ through the \ac{ae} and the nominal acceleration $\ddtildebmq_{\text{nom}}$ on the embedded submanifold induced from the control torque $\tildebmtau_{\text{c}}$ in the nominal projected dynamics~\eqref{eq:projectedequationsofmotion}.
We use~\eqref{eq:tilde_projected_plant_acceleration} and ~\eqref{eq:projectedequationsofmotion} to rewrite the expanded difference $\tilde{\bmM}\ddtildebmq = \tildebmM\ddtildebmq_{\text{nom}} + \tildebmM(\ddtildebmq - \ddtildebmq_{\text{nom}})$ as
\begin{align*}
    \label{eq:projected_cl_dynamics_app}
    \tildebmM\tildebmq =& (\tildebmtaudes - (\tildebmC + \tildebmD)\tildebmq - \tildebmg) \nonumber\\
    &+ \tildebmM\Big(d\varphiq d\rhoq \bm M^{-1}(\tildebmtaudes -(\bm C+\bm D)\dbmq -\bm g) \nonumber \\
    &\quad+d\varphiq d^2\rhoq\dbmq^2 + d^2\varphiq \dcheckbmq^2 \\
    &\quad-\tildebmM^{-1}(\tildebmtaudes - (\tildebmC + \tildebmD)\tildebmq - \tildebmg)\Big) \nonumber.
\end{align*}
Substituting $\tildebmtau_{\text{c}}$ by~\eqref{eq:torque_reconstruction} with $\checkbmtau_{\text{c}}$ from~\eqref{eq:reduced_pd_ff_controllaw} and using $\ddtildebmq\!=\! \ddtildebme + \ddtildebmqdes$, $\dtildebmq\!=\! \dtildebme + \dtildebmqdes$, and $\dcheckbmq = d\rhoq \dbmq$, we obtain

\small
\begin{equation}
\begin{aligned}
\label{eq:tilde_closed_loop_expanded_full}
\vspace{-0.5cm}
    \checkbmM(\ddtildebme+ \ddtildebmqdes) = 
    & \Big(d\rhoq^\intercal \big(\check{\bm M}_{\bm\theta}\ddcheckbmqdes + (\check{\bm C}_{\bm\theta} + \check{\bm D}_{\bm\theta})\dcheckbmqdes + \check{\bm g}_{\bm\theta} - \checkKp\check{\bm e} - \checkKd \dot{\check{\bm e}} \big) \\ 
    & \quad\quad - (\tildebmC + \tildebmD) (\dtildebme + \dtildebmqdes) - \tildebmg \Big) \\
    &+ \tildebmM d\varphiq d\rhoq \bmM^{-1}\Big(\tildebmtaudes - (\bm{C} + \bm{D}) \dbmq - \bm{g}\Big) \\ 
    &+ \tildebmM d\varphiq d^2\rhoq \dbmq^2 + \tildebmM d^2\varphiq d\rhoq^2\dbmq^2 \\ 
    & - \Big(\tildebmtaudes - (\tildebmC + \tildebmD) \dtildebmq - \tildebmg \Big).
\end{aligned}
\end{equation}
\normalsize
Notice that, in contrast to the latent closed-loop dynamics of Proposition~\ref{prop:latent-closedloop} and App.~\ref{appendix:latent_closedloop_dynamics}, the expression depends on the  curvature induced by the decoder via $d^2\varphiq$ as pointed out in Remark~\ref{remark:decoder-curv}. 

The closed-loop dynamics on the embedded submanifold can be rearranged into the form~\eqref{eq:tilde_closed_loop_dyns} with the dynamic modeling disturbance
\begin{equation}
\begin{split}
        \tilde\Delta_{\theta} &= \big(d\rhoq^\intercal (\check{\bm M}_{\bm\theta}\ddcheckbmqdes + (\check{\bm C}_{\bm\theta} + \check{\bm D}_{\bm\theta})\dcheckbmqdes + \check{\bm g}_{\bm\theta} ) \\ 
    & \quad\quad - \tildebmM \ddtildebmqdes- (\tildebmC + \tildebmD) \dtildebmqdes - \tildebmg \big),
\end{split}
        \label{eq:tilde_delta_theta_dynamic_modeling_error}
\end{equation}
and the projection alignment disturbance
\begin{equation}
    \label{eq:tilde_delta_perp_projection_alignment_error}
    \begin{split}
    \tilde\Delta_{\perp} =& \tildebmM d\varphiq d\rhoq \bmM^{-1} \tildebmtaudes - \tildebmtaudes \\
    &+ \tildebmM d\varphiq d^2\rhoq \dbmq^2 + \tildebmM d^2\varphiq d\rhoq^2\dbmq^2  \\
    &+(\tildebmC + \tildebmD) \dtildebmq - \tildebmM d\varphiq d\rhoq \bmM^{-1} (\bm{C} + \bm{D}) \dbmq \\
    &+ \tildebmg - \tildebmM d\varphiq d\rhoq \bmM^{-1}\bm{g}.
    \end{split}
\end{equation}
The damping gain $\tildeKd$ in~\eqref{eq:tilde_closed_loop_dyns} is obtained from~\eqref{eq:tilde_closed_loop_expanded_full} by identifying
\begin{equation}
    \tildeKd \dtildebme = d\rhoq^\intercal\checkKd \dcheckbme.
\end{equation}
Substituting $\dcheckbme = d\rhoq \dtildebme$, we obtain
\begin{equation}
    \label{eq:projKd_from_checkKd}
    \tildeKd = d\rhoq^\intercal\checkKd d\rhoq.
\end{equation}
The proportional gain $\tildeKp$ in~\eqref{eq:tilde_closed_loop_dyns} is similarly identified from~\eqref{eq:tilde_closed_loop_expanded_full} as
\begin{equation}
    \label{eq:tildekp_rel}
    \tildeKp \tildebme = d\rhoq^\intercal\checkKp \checkbme.
\end{equation}
However, in contrast to its derivative, the projected error is nonlinearly related to the latent error, i.e., $\checkbme = \rhoq (\tildebme)$. 
This implies that the potential gain $\tildeKp$ follows a quadratic form of the type~\eqref{eq:projKd_from_checkKd} only for linear reductions. For nonlinear reductions,~\eqref{eq:tildekp_rel} does not guarantee that $\tildeKp$ is positive definite for all $\tildebmq$.
However, $\tildeKp$ can be approximated as a positive-definite quadratic form for general nonlinear reductions by linearizing the error $\checkbme$ around the origin $\tildebmqdes$, i.e., 
\begin{equation}
\label{eq:linearization_latent_error_for_kp_tilde}
    \checkbme = \rhoq(\tildebmq) - \rhoq(\tildebmqdes) = d\rhoq \tildebme + \mathcal O(\|\tildebme\|^2).
\end{equation}
Substituting the linearized error~\eqref{eq:linearization_latent_error_for_kp_tilde} in~\eqref{eq:tildekp_rel}, we deduce as first order approximation of the potential gain as
\begin{equation}
\label{eq:approx_tildeKp}
    \tildeKp \approx d\rhoq^\intercal\check~Kp d\rhoq.
\end{equation}

\subsection{Proof of Proposition~\ref{prop:tilde_closed_loop} --- Part 2: Derivation of the Convergence Bounds on the Embedded Submanifold}
\label{appendix:embedded_convergence_bounds}
Next, we provide bounds on the disturbances $\tilde\Delta_{\theta}$ and $\tilde\Delta_{\perp}$ and discuss the conditions under which the \ac{rolnn} exactly models the nominal closed-loop dynamics on the embedded submanifold.

\subsubsection{Dynamic Modeling Disturbance}
First, we rearrange~\eqref{eq:tilde_delta_theta_dynamic_modeling_error} by substituting $\ddcheckbmqdes = d\rhoq \ddtildebmqdes + d^2\rhoq\dtildebmqdes^2$ and $\dcheckbmqdes = d\rhoq \dtildebmqdes$, 
\begin{equation}
    \label{eq:tilde_delta_theta_simplified}
    \begin{split}
        \tilde\Delta_\theta &= (d\rhoq^\intercal \checkbmMtheta d\rhoq - \tildebmM) \ddtildebmqdes \\
        &+ d\rhoq^\intercal \checkbmMtheta d^2\rhoq \dtildebmqdes^2\\
        &+ d\rhoq^\intercal \checkbmCtheta \dcheckbmqdes- \tildebmC\dtildebmqdes\\
        &+ (d\rhoq^\intercal \checkbmDtheta d\rhoq - \tildebmD) \dtildebmqdes \\
        &+ (d\rhoq^\intercal \checkbmgtheta - \tildebmg)
    \end{split}
\end{equation}
We can then derive the projected Coriolis force analogously to~\eqref{eq:pullback_coriolis} using the projection $\bm P$ instead of $d\varphiq$ and its derivative
\begin{equation}
\label{eq:proj_hessian}
    d\bm P = d\varphiq d^2\rhoq +d^2\varphiq d\rhoq^2,
\end{equation} 
yielding
\begin{equation}
\label{eq:pullback_coriolis_tilde}
    \tildebmC\dtildebmqdes = \bm P^\intercal \bm C \bm P \dtildebmqdes + \bm P^\intercal \bm M d\varphiq d^2\rhoq \dtildebmq \dtildebmqdes + \bm P^\intercal \bm M d^2\varphiq d\rho^2 \dtildebmq\dtildebmqdes.
\end{equation} 

Next, we discuss the conditions under which $\tilde\Delta_\theta=0$.
We assume that the latent \ac{lnn} perfectly converges, i.e., it learns the pullback dynamic parameters~\eqref{eq:pullback_dynamic_parameters},\eqref{eq:pullback_dissipation}. Then, using~\eqref{eq:pullback_coriolis},~\eqref{eq:pullback_coriolis_tilde}, and the projected parameters of~\eqref{eq:projectedequationsofmotion}, and again substituting $\ddcheckbmqdes = d\rhoq \ddtildebmqdes + d^2\rhoq\dtildebmqdes^2$ and $\dcheckbmqdes = d\rhoq \dtildebmqdes$, the disturbance in~\eqref{eq:tilde_delta_theta_simplified} simplifies to 
\begin{equation*}
\tilde\Delta_\theta=d\rhoq^\intercal \checkbmM d^2\rhoq \dtildebmqdes^2 +\bm P^\intercal \bm M d^2\varphiq d\rhoq^2\dtildebmq \dtildebmqdes - \bm P^\intercal \bm M d\bm P \dtildebmq \dtildebmqdes,
\end{equation*}
which further simplifies using~\eqref{eq:proj_hessian} to 
\begin{equation*}
    \tilde\Delta_\theta=d\rhoq^\intercal \checkbmM d^2\rhoq \dtildebmqdes^2-\bm P^\intercal \bm M d\varphiq d^2\rhoq \dtildebmq\dtildebmqdes.
\end{equation*}
Using~\ref{eq:app_simplification_d2_rho}, and the definition of the pullback mass matrix, the disturbance finally simplifies to 
\begin{equation*}
    \tilde\Delta_\theta=-d\rhoq^\intercal\checkbmM d^2\rhoq \dtildebmqdes \dtildebme.
\end{equation*}
Importantly, the disturbance $\tilde\Delta_\theta$ only vanishes for $\dtildebme=0$ or $\dtildebmqdes=0$. In general, even the \ac{rolnn} perfectly recovers the nominal latent dynamic parameters, the disturbance remains due to the curvature of the encoder, i.e., $d^2\rhoq$. In other words, it vanishes only for linear point reduction maps $\rhoq$ for which $d^2\rhoq=0$.

Finally, following the same arguments as in App.~\ref{appendix:latent_convergence_bounds}, we conclude that, for any $(\bm q, \dbmq)\in \mathcal N$ and bounded reference trajectories, the dynamic modeling disturbance is bounded, i.e., there exists a constant $\check r_\theta>0$ such that ${\|\tilde\Delta_\theta (\bmq, \dbmq, \bmqdes, \dbmqdes, \ddbmqdes)\|\leq \tilde r_\theta}$.

\subsubsection{Projection Alignment Disturbance}
We start from the disturbance equation in the form~\eqref{eq:tilde_delta_perp_projection_alignment_error}. 
Using the idempotence of the projection $\bm P \bm P \!=\! \bm P$,~\eqref{eq:proj_hessian}, $\dtildebmq = \bm P \dbmq$, and substituting the projected pullback Coriolis forces $\tildebmC\dtildebmq$ as in~\eqref{eq:pullback_coriolis}, the disturbance is rewritten as
\begin{equation}
\label{eq:app_tilde_delta_perp_simplified_generally}
\begin{split}
    \tilde\Delta_\perp &= \tildebmM\bm M^{-1} \tildebmtaudes - \tildebmtaudes + \tildebmM d\bm P \dbmq^2 \\
    & + \bm P^\intercal \bm M d\bm P \dtildebmq^2 +(\bm P^\intercal \bm C \bm P - \tildebmM\bmM^{-1}\bm C)\dbmq \\
    & + \tildebmD\dtildebmq  - \tildebmM \bmM^{-1} \bm D \dbmq\\
    &+ \tildebmg - \tildebmM\bmM^{-1} \bm g.
\end{split}
\end{equation}

\textbf{$\tilde\Delta_\perp$ for a $\bm M$-orthogonal projection.}
In the case where the AE learns a $\bmM$-orthogonal projection, we have ${d\rhoq \!=\! d\varphiq^{+_{,\bm M}}}$ and we substitute $d\rhoq$ by the generalized inverse~\eqref{eq:riemannian_pseudoinverse_of_embedding} in~\eqref{eq:app_tilde_delta_perp_simplified_generally}.
In this specific case, $\tildebmM\bmM^{-1} \!=\! \bm M d\varphiq\checkbmM^{-1}d\varphiq^\intercal = \bm P^\intercal$, and thus, the expression~\eqref{eq:app_tilde_delta_perp_simplified_generally} vanishes on the embedded submanifold. Note that the projection alignment disturbance generally remains nonzero outside of the submanifold, similarly as the dynamic modeling disturbance.

\textbf{General bound on $\tilde\Delta_\perp$.} 
The bound on $\tilde\Delta_\perp$~\eqref{eq:tilde_delta_perp_projection_alignment_error} is derived similarly as the bound on $\check\Delta_\perp$. As the derivation follows the same procedure as in App.~\ref{appendix:latent_convergence_bounds}, we do not provide all details here but simply summarize the results. Under the same considerations as before, each of the terms in~\eqref{eq:tilde_delta_perp_projection_alignment_error} is bounded by definition on a compact state space and we can establish bounds depending on the angle $\sfrac{\pi}{2}-\alpha$ between $\bm P$ and $\bm P_\perp$. 
This allows us to conclude that, for a \ac{rolnn} achieving a low loss~\eqref{eq:loss_RO_LNN}, and for any $\forall (\bm q, \dbmq)\in \mathcal N$ and bounded reference trajectories, the projection alignment disturbance is small and bounded, i.e., there exists a constant $\tilde r_\perp>0$ such that $\|\tilde\Delta_\perp (\bmq, \dbmq, \bmqdes, \dbmqdes, \ddbmqdes)\| \leq \tilde r_\perp$.

\subsection{Proof of Theorem~\ref{thm:projectedISS} --- Stability and Convergence on Embedded Submanifold Proof}
\label{appendix:embedded_iss_proof}
We now discuss the proof for Theorem~\ref{thm:projectedISS}, showing that the closed-loop dynamics~\eqref{eq:tilde_closed_loop_dyns} on $\varphi(\mathcal{T}\checkmathcalq)$ under disturbance are locally exponentially stable. As the proof follows similar steps as these of Corollary~\ref{cor:latent_system_converges} in App.~\ref{appendix:latent_closed_loop_errorfree_proof} and Theorem~\ref{thm:latentISS} in App.~\ref{appendix:latent_iss_proof}, we only provide a summary of the proof.

Similarly as Corollary~\ref{cor:latent_system_converges} for Theorem~\ref{thm:latentISS}, we start by discussing the stability and convergence of the disturbance-free dynamics. Specifically, the nominal closed-loop dynamics on the submanifold $\varphi(\mathcal{T}\checkmathcalq)$, i.e., the dynamics~\eqref{eq:tilde_closed_loop_dyns} with optimally-converged \ac{rolnn} parameters, are locally exponentially stable. 
This is proven by considering the Lyapunov candidate 
\begin{equation}
\label{eq:lyapunov_function_tilde}
    \vartheta(\checkbmx)=\frac{1}{2} \tildebmx^\intercal \bm \Theta\tildebmx \enspace \text{with} \enspace \bm \Theta = \begin{pmatrix}\tildeKp & \epsilon \tildebmM\\ \epsilon \tildebmM  &\tildebmM \end{pmatrix}.
\end{equation} 
and following the same steps as the proof of Corollary~\ref{cor:latent_system_converges} in App.~\ref{appendix:latent_closed_loop_errorfree_proof}. 
Then, the proof of Theorem~\ref{thm:projectedISS} is obtained by including the disturbances $\tilde\Delta_\theta$, $\tilde\Delta_\perp$ in the time derivative of the Lyapunov candidate similar to~\eqref{eq:lyapunov_derivative_simple_disturbed} and following the same steps as the proof of Theorem~\ref{thm:latentISS} in App.~\ref{appendix:latent_iss_proof}.

Note that the closed-loop dynamics on $\varphi(\mathcal{T}\checkmathcalq)$ are locally exponentially ISS even if the proportional gain $\tildeKp$ is strictly identified via~\eqref{eq:tildekp_rel} instead of being approximated as in~\eqref{eq:approx_tildeKp}. This is proven by considering the pullback potential 
\begin{equation}
\label{eq:pullback_potential_positiongain_V_p_tilde}
     \tilde V_{\text{P}} (\tildebmq) = \frac{1}{2}\big(\rhoq(\tildebmq)-\rhoq(\tildebmqdes)\big)^\intercal\checkKp \big(\rhoq(\tildebmq)-\rhoq(\tildebmqdes) \big)
\end{equation}
in the Lyapunov candidate 
\begin{equation}
    \vartheta(\tildebmx) = \tilde V_{\text{P}} (\tildebme + \tildebmqdes) + \epsilon\dtildebme^\intercal \tildebmM \tildebme + \frac{1}{2}\dtildebme^\intercal \tildebmM \dtildebme.
\end{equation}
and following the same steps as the original proof, leveraging the fact that the derivative of~\eqref{eq:pullback_potential_positiongain_V_p_tilde} simplifies the same way as that of a linearly-dependent potential.
Note that showing the negative-definiteness of $\dot\vartheta$ then requires $\tildebme^\intercal d\rhoq^\intercal \checkKp\checkbme\succ 0$. This can be simplified by using the linearization~\eqref{eq:linearization_latent_error_for_kp_tilde} to approximate the derivative of the potential $\tilde V_{\mathrm{P}}$ as
\begin{equation}
\label{eq:approx_dKp}
    \frac{\partial}{\partial \tildebmq}\tilde V_{\mathrm{P}}(\tildebmq) 
    = d\rhoq^\intercal\checkKp d\rhoq \tildebme + \mathcal{O}(\|\tildebme\|^2).
\end{equation}
From~\eqref{eq:approx_dKp}, we have $\tildebme^\intercal d\rhoq^\intercal \checkKp\checkbme \!=\! \tildebme^\intercal d\rhoq^\intercal\checkKp d\rhoq \tildebme + \mathcal{O}(\|\tildebme\|^3)$. Moreover, as the pullback through the encoder Jacobian preserves positive-definiteness, there exists a positive $k_p$ such that $\tildebme^\intercal d\rhoq^\intercal \checkKp\checkbme\geq k_p \|\tildebme\|^2$.
In summary, if the gain $\checkKp$ is high, the forward invariance neighborhood $\Omega_0$ come become large.

\subsection{Control Gain Design}
\label{appendix:gains}
In this section, we briefly discuss the design of the latent gains $\checkKp$ and $\checkKd$ of the \ac{rolnn}-based control law. 

For simplicity, control gains are often set as constant. In our setup, this would correspond to setting constant gains $\Kp$, $\Kd$ in the original, high-dimensional configuration space $\mathcalq$. Due to the structure of the \ac{rolnn}, the corresponding latent gains are $\checkKp = d\varphiq^\intercal \Kp \varphiq(\rhoq(\bm q)-\rhoq(\bmqdes))$ and $\checkKd = d\varphiq^\intercal \Kd d\varphiq$.
For simplified computation during the deployment of our controller, we instead set $\checkKp$ and $\checkKd$ as constant diagonal matrices, which corresponds to selecting position-dependent full-order gains $\Kp$ and $\Kd$. 

Setting good values for the gains $\checkKp$ and $\checkKd$ is challenging due to the \ac{ae} mappings, whose magnitude depends on the training of the \ac{rolnn}, thus compromising standard intuition. In practice, we first determined the latent $\checkKd$ by pulling back a constant full-order gain $\Kd$ across the training dataset. We empirically fixed the magnitude of the diagonal entries of $\checkKd$ based on the ratio of the eigenvalues of the pullback gain $d\varphiq^\intercal \Kd d\varphiq$. Notice that we observed only small variation across the training datasets, supporting the hypothesis that a constant $\checkKd$ is a reasonable gain. Finally, we set $\checkKp$ empirically based on $\checkKd$ and fine-tuned both gains. Notice that thanks to the Lagrangian structure of the \ac{rom} in the learned latent space, the gains still fulfill standard assumptions. For instance, a larger $\checkKp$ leads to a more aggressive tracking behavior and can lead to instability, while the damping gain $\checkKd$ is typically chosen smaller than $\checkKp$.

\subsection{Additional Details on Datasets and Models}
\label{appendix:datasets-models}

\subsubsection{Simulated Augmented Pendulum from Sec.~\ref{sec:experiments:pendulum}} 
\label{appendix:datasets-models-pendulum} 
We use the $15$-\ac{dof} augmented pendulum introduced as a benchmark system in~\cite[Sec.4.1]{friedl2025hamiltonian}, composed of a $3$-link planar pendulum and a $12$-\ac{dof} planar mass-spring-damper, as depicted in Fig.~\ref{fig:sketch_bd_pend}. The pendulum and mesh parameters are as in~\cite{friedl2025hamiltonian}, with additional damping coefficients set to $0.1$ for the pendulum \acp{dof}, and $0.01$ for the mesh \acp{dof}. All simulations are carried out in \textsc{Mujoco}~\citep{Todorov12:mujoco}, where we have access to the ground truth dynamics.

The two subsystems are fully decoupled, i.e., there is no interaction between the masses of the pendulum and the masses of the mesh. Note that the physical parameters of the subsystems were selected so that the amplitude of the pendulum is much larger amplitude than that of the mesh. Therefore, the latter acts as a noise, or in other world, as a systematic measurement error, implying that a \ac{rom} based solely on the pendulum \acp{dof} would capture the dominant behavior of the full system. In other words, the system is well-reducible via a \ac{rolnn}. To ensure non-trivial reducibility, 
we specify a diffeomorphism $h\!:\!\mathcal Q_{\text{aug}} \!\to \!\mathcal Q$ from the augmented configuration vector  $\bm q_{\text{aug}}\!=\!\left(\bm q_{\text{pend}}^\intercal, \bm q_{\text{ms}}^\intercal\right)^\intercal$ to a nonlinearly-transformed representation $\bmq \!=\! h(\bmq_{\text{aug}})$, similar to~\cite[App.G.1]{friedl2025hamiltonian}.

\begin{figure}
    \centering
    \adjustbox{trim=0cm 0.0cm 0cm 0.0cm}{
            \includesvg[width=.95\linewidth]{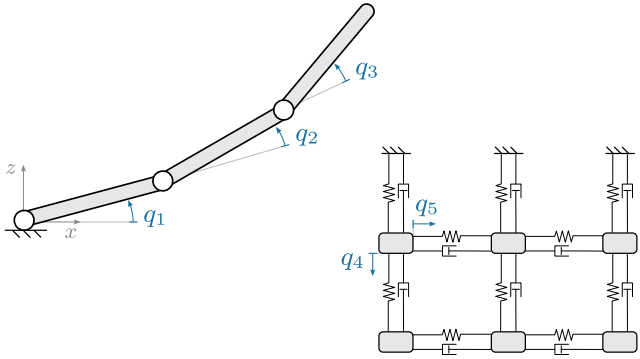}  
        }
    \caption{Sketch of the augmented pendulum system of Sec.~\ref{sec:experiments:pendulum}.
    }
    \label{fig:sketch_bd_pend}
\end{figure}

\textbf{Training Data.}
The initial configurations and velocities of each \ac{dof} of the pendulum are randomly sampled from the intervals $q_{\text{pend}, i}(t\!=\!0) \!\in\! \left[-\frac{\pi}{2}, \frac{\pi}{2}\right]\SI{}{\degree}$ and ${\dot q_{\text{pend}, i}(t\!=\!0) \!\in\! \left[-23, 23\right]\SI{}{\degree\per\second}}$. 
Initial displacements and velocities for each mesh-\ac{dof} are randomly sampled from the intervals $q_{\text{ms}, j}(t\!=\!0) \!\in \! \left[-1, 1\right]\times10^{-2}\SI{}{\meter}$ and $\dot q_{\text{ms}, j}(t\!=\!0) \!\in\! \left[-2, 2\right]\times 10^{-3}\SI{}{\meter\per\second}$.
We record an initial training dataset of $10$ trajectories with control input $\bmtaudes=\bm 0$. 
To ensure a more general learning of the dynamics, we augment this initial training dataset with more diverse trajectories, namely $10$ trajectories recorded on a regulation task, and $10$ trajectories recorded on a sine-tracking task. 
For the regulation task, we randomly sample $\bmqdes$ from the same distributions as for the initial conditions. For the sine-tracking task, we sample sinusoidal reference trajectories for each joint $i$ such that the reference trajectories are set to track signals 
$q_{\text{ref},i} \!=\! A_i \sin(2\pi f_i t + \phi)$, with $\phi_i \!=\! \arcsin\left(\frac{q_{i,0}}{A_i}\right)$.
For the $i=\{1,2, 3\}$ pendulum joints, we sample values with $A_i \in \left[1,40\right]\SI{}{\degree}$, $f_i\in\left[\frac{1}{10}, 3\right]\SI{}{\per\second}$.
For the mesh joints $i=4,...,15$, we sample $A_i \in \left[-1,1\right]\times 10^{-2}\SI{}{\meter}$, $f_i\in\left[\frac{1}{5}, 3\right]\SI{}{\per\second}$. 
The references $\dot{q}_{\text{ref},i}$ and $\ddot{q}_{\text{ref},i}$ are obtained as the first and second time derivatives of $q_{\text{ref},i}$. 
To generate $\bmtaudes$ in both controlled scenarios, we use~\eqref{eq:pd_ff_controllaw} and the analytic model.

Each simulation is recorded for $T=\SI{5}{\s}$ at a timestep of $\Delta t = 10^{-3}\SI{}{\second}$, yielding $N=30$ training trajectories $\mathcal{D}= \{\{h(\bmq_{\text{pend}, n,k}, \bmq_{\text{ms}, n,k})^\intercal, dh(\bmq_{\text{pend}, n,k}, \bmq_{\text{ms}, n,k})^\intercal\}_{k=1}^{K}\}_{n=1}^{N}$ with $K=5000$ observations each.

\textbf{\ac{rolnn} Training Details.}
We train a geometric \ac{rolnn} composed of a biorthogonal \ac{ae} and a latent geometric \ac{lnn}. 

 We use $l=\{1,2,3\}$ pairwise biorthogonal encoder and decoder layers of sizes $n_l = \{6, 12, 15\}$ with latent space dimension $n_0=3$. The biorthogonal weight matrices are initialized by sampling a random orthogonal matrix $\bm O \in \mathbb{R}^{n_l\times n_l}$ from the Haar distribution and setting $\bm\Phi=\bm\Psi=\bm O_{[:,:n_{l-1}]}$, where $\bm O_{[:,:n_{l-1}]}$ are the first $n_{l-1}$ column entries of $\bm O$. Bias vectors are initialized as $\bm{b}_{l}=\bm{0}$.
For the latent geometric \ac{lnn}, we parametrize the potential energy network $\check V_{\bm\theta_{\check{V}}}$ as well as the Euclidean parts $g_{\mathbb{R}}$ of the mass-inertia matrix network $\check{\bmM}_{\bm\theta_{\check{T}}}$ and damping network $\check{\bm D}_{\theta_{\check{\mathcal D}}}$ each with $L_{\check{V}}\!=\!L_{\check T, \mathbb R}=2$ hidden Euclidean layers of $32$ neurons and SoftPlus activation functions. We fix the basepoint of the exponential map layers $g_{\text{Exp}}$ in the mass-inertia matrix network $\check{\bmM}_{\bm\theta_{\check{T}}}$ and damping network $\check{\bm D}_{\theta_{\check{\mathcal D}}}$ to the origin $\bm P\!=\!\bm I$. Weights are initialized by sampling from a Xavier normal distribution with gain $\sqrt{2}$ and bias vector entries set to $1$.
We train the model on $4000$ uniformly sampled random points from the dataset $\mathcal{D}$ using the loss~\eqref{eq:loss_RO_LNN}  with Runge-Kutta $4$-th order integration over a training horizon of $H_{\train}\!=\!10$ timesteps. We use a learning rate of $1.5\times 10^{-2}$ for the \ac{ae} parameters and $7\times 10^{-4}$ for the \ac{lnn} parameters. We train the model with Riemannian Adam~\citep{becigneul2018riemannianoptimization} until convergence at $2000$ epochs.

\textbf{Regulation Experiments.}
The controller gains for the \ac{rolnn}-based law~\eqref{eq:reduced_pd_ff_controllaw} are defined following App.~\ref{appendix:gains} by setting the diagonal entries of the latent gain matrices as $(400, 1800, 600)$ for $\checkKp$  and $(60, 270, 90)$ for $\checkKd$.
For the model-free PD baseline, we set the first three diagonal entries of $\Kp$ to $(25, 15, 10)$ and $\Kd$ to $(2, 1.5, 1)$, and the latter $12$ entries uniformly to $1.5$ for $\Kp$ and $0.15$ for $\Kd$.
For the pendulum \acp{dof} in the linear-intrusive baseline, we set $\checkKp$ and $\checkKd$ equal to the first three entries in the model-free PD case. 

\textbf{Trajectory Tracking Experiments.}
As controller gains for the \ac{rolnn}-based law~\eqref{eq:reduced_pd_ff_controllaw}, we set the diagonal entries of the latent gain matrices to $(300, 800, 1500)$ for $\checkKp$ and $(30, 80, 150)$ for $\checkKd$.
For the model-free PD baseline and for the hybrid controller~\ref{eq:fom_pd_lat_ff_controllaw}, we set the first three diagonal entries of the gain matrix $\checkKp$ to $(150, 100, 50)$ and $\checkKd$ to $(6, 2, 2)$. We set the last $12$ entries uniformly to $1.5$ for $\Kp$ and $0.1$ for $\Kd$.
For the pendulum \acp{dof} in the linear-intrusive baseline, we set $\checkKp$ and $\checkKd$ equal to the first three entries in the model-free PD case.

We generate $10$ circular reference trajectories for the pendulum \acp{dof} by randomly sampling center points for the circle with $x\in[0.5, 1.0]\SI{}{\meter}$ and $z\in[-1.0, 1.0]\SI{}{\meter}$, with radii drawn from $[0.25, 0.8]\SI{}{\meter}$. Tracking frequencies were sampled $[0.5, 3]$Hz, with angular tracking velocity $\omega=2\pi f$. The mesh \acp{dof} follow randomly sampled sine references, following the process described above for training data generation.

To map the taskspace references to the pendulum \acp{dof} in joint space, we use a damped least-squares inverse kinematics solver to obtain joint positions, the Jacobian pseudoinverse to compute joint velocities from the analytical task-space velocities, and the Jacobian derivative to analytically compute joint accelerations via $\ddot{\bm q} = \bm J^+ (\ddot{\bm x} - \dot{\bm J} \dot{\bm q})$. 

\subsubsection{Underactuated Puppet from Sec.~\ref{sec:experiments:monkey}}
We use a plush monkey puppet suspended by strings. The puppet is actuated via the left arm of the RB-Y1 humanoid robot and observed using a ZED Mini RGB-D camera mounted on the robot’s head. Although the monkey’s shape formally lies in an infinite-dimensional space, we approximate it using $117$ 3D points, resulting in a $351$-\acp{dof} system. The complete real-world experimental setup, including the robot, the camera, and the point-based representation of the monkey described shortly, is shown in Fig.~\ref{fig:monkey:setting}.

\begin{figure}[]
    \centering
    \includegraphics[width=\linewidth]{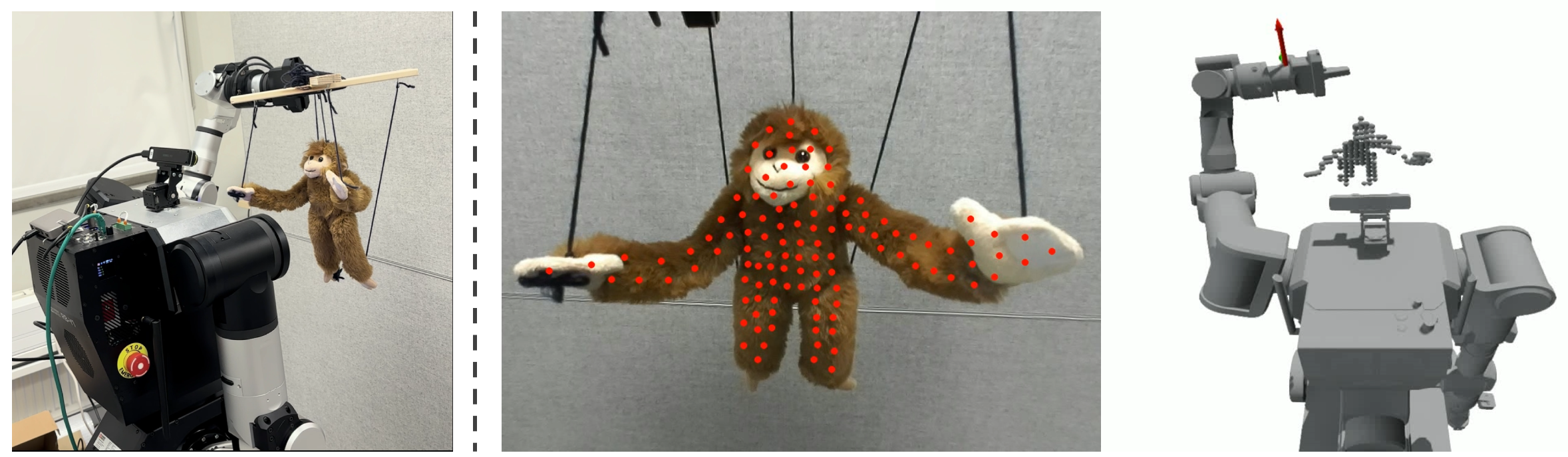}
    \caption{\emph{Left}: Real-world robot environment setting. \emph{Middle-right}: Initial 2D and 3D points representing the evolution of the plush monkey puppet.
    }
    \label{fig:monkey:setting}
\end{figure}

We designed a perception pipeline that extracts 3D points and ensures consistent point tracking throughout the monkey’s motion. The inputs to the pipeline are {RGB-D} images, and the output are the 3D coordinates of points with consistent correspondence over time. 

Before initiating the real-time perception, we initialize the monkey’s points in pixel space using a segmentation mask. First, we employ the pretrained segmentation network SAM~\cite{kirillov2023segany} with the language prompt ``monkey plush toy'' to obtain a segmentation mask of the monkey. To improve the stability of optical flow estimation (described in the next paragraph), we apply morphological operations to the mask to exclude pixels near the monkey's boundary. We then downsample the remaining mask pixels to construct the initial set of 2D points representing the monkey. Pixels whose depth values are NaN in the depth image are also removed. The remaining pixels are back-projected using the depth image to obtain their corresponding 3D points --- in our case, the $117$ 3D points shown in Fig.~\ref{fig:monkey:setting}-\emph{right}. Next, we describe how we track the initialized points throughout the video. 

When the RGB-D image at the next time step is observed, we first estimate the $2$-dimensional optical flow between the previous and current RGB images using the pretrained flow estimation model RAFT~\cite{teed2020raft}. Then, we refine the estimated optical flow using the segmentation mask of the monkey in the current frame, obtained via SAM. More specifically, this refinement is performed by optimizing the Chamfer distance between the 2D pixel coordinates propagated by the optical flow and those constrained by the segmentation mask. Using the refined optical flow, we obtain the monkey’s 2D pixel coordinates at the next time step. These coordinates are subsequently back-projected into three-dimensional space using the corresponding depth image, yielding the updated 3D point positions.

The control input for the monkey system is the desired angular velocity vector of the end-effector, denoted by ${\bm u = (r_x \; r_y \; r_z)^\intercal \in \mathbb{R}^3}$, where $r_x, r_y$, and $r_z$ represent the angular velocities about the $x$, $y$, and $z$-axis of the end-effector frame, respectively, expressed in the end-effector frame. The real-world robot follows the desired angular velocity via a low-level feedback controller.

\begin{figure}[]
    \centering
    \includegraphics[width=\linewidth]{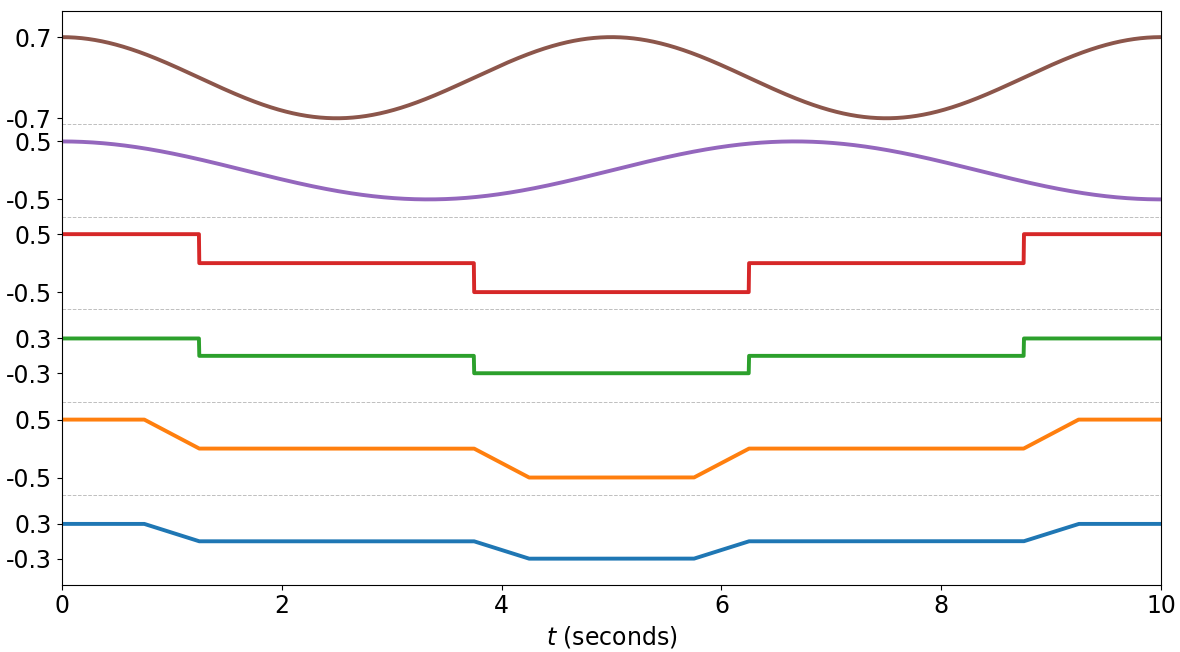}
    \caption{Six speed profiles determining the scalar magnitude of velocity control trajectories for the plush monkey puppet dataset.
    }
    \label{fig:monkey:speed_profiles}
\end{figure}

\textbf{Training Data.}
We collect a total of $18$ trajectories of $10$ seconds each composed of the monkey’s 3D points trajectories under $18$ distinct desired angular velocity control trajectories. The initial configuration of the monkey’s points is fixed across all trajectories, and the initial velocities are set to zero. The desired control trajectories are generated by combining three fixed angular directions with six time-varying speed profiles. Specifically, for the angular directions $V_i \in \mathbb{R}^3$, $i=1, 2, 3$, we use
\begin{equation}
    \label{eq:actuationprofiles}
    \begin{split}
        V_1 &= (1/\sqrt{2}, 1/\sqrt{2}, 0), \\
        V_2 &= (0, 1/\sqrt{2}, 1/\sqrt{2}), \\
        V_3 &= (1/\sqrt{2}, 0, 1/\sqrt{2}).\\
    \end{split}
\end{equation}
For the speed profiles, we define functions $s_j: \mathbb{R}_{\geq0} \rightarrow \mathbb{R}$, $j=1, \ldots, 6$, which map time to a scalar magnitude. We use six distinct periodic functions, consisting of cosine-shaped profiles and gated square waveforms, shown in Fig.~\ref{fig:monkey:speed_profiles}. For all combinations of angular directions and speed profiles, i.e., ${\bm u(t)} = s_j(t)V_i$, we obtain 18 distinct control trajectories and the corresponding monkey motions, represented as trajectories of 3D points.
As the training trajectories were obtained at a framerate of $4$Hz, the training dataset was interpolated to train on a timestep of $\Delta t = 10^{-3}\SI{}{\second}$.

\textbf{\ac{rolnn} Training Details.}
We train a geometric \ac{rolnn} composed of a biorthogonal \ac{ae} and the latent geometric \ac{lnn} on $2000$ uniformly-sampled point across the $18$ training trajectories. 
We use $l=\{1,2,3, 4\}$ pairwise biorthogonal encoder and decoder layers of sizes $n_l = \{32, 64, 128, 351\}$ with latent space dimension $n_0=6$. 
For the latent geometric \ac{lnn}, we parametrize the potential energy network $\check V_{\bm\theta_{\check{V}}}$ and the Euclidean parts $g_{\mathbb{R}}$ of the mass-inertia network $\check{\bmM}_{\bm\theta_{\check{T}}}$ and damping network $\check{\bm D}_{\theta_{\check{\mathcal D}}}$ each with $L_{\check{V}}\!=\!L_{\check T, \mathbb R}\!=\!2$ hidden Euclidean layers of $32$ neurons and SoftPlus activation functions. We fix the basepoint of the exponential map layers $g_{\text{Exp}}$ in the mass network $\check{\bmM}_{\bm\theta_{\check{T}}}$ and damping network $\check{\bm D}_{\theta_{\check{\mathcal D}}}$ to the origin $\bm P\!=\!\bm I$. The network weights are initialized the same as for the \ac{rolnn} of the pendulum as detailed in App.~\ref{appendix:datasets-models-pendulum}.

To model the actuation matrix $\bm B_\theta(\bm q)$ as detailed in Sec.~\ref{sec:method_part2}, we employ another fully-connected MLP with $2$ hidden Euclidean layers of $64$ neurons and SoftPlus activation functions. Parameters are initialized as for the other MLP parameters. Note that the network output is a relatively large matrix with $n\times m$ entries. As the underactuated robotic system that we consider remain of reasonable dimension, this is not an issue in practice. In our case, the linear output layer has $n\times m = 1053$ entries and training worked out-of-the-box without additional loss considerations. 

We train the model on $2000$ uniformly-sampled random points from the dataset $\mathcal{D}$ using the loss~\eqref{eq:loss_RO_LNN} with a Runge-Kutta $4$-th order integration over a training horizon of $H_{\train}\!=\!10$ timesteps. We use a learning rate of $1.5\times 10^{-2}$ for the \ac{ae} parameters and $7\times 10^{-4}$ for the \ac{lnn} parameters. We train the model with Riemannian Adam~\citep{becigneul2018riemannianoptimization} until convergence at $4800$ epochs.

\textbf{Regulation Experiments}
We first generate $3$ unseen reference configurations of the monkey. Since the monkey puppet control system is inherently underactuated, we must ensure that the reference configurations are executable and reachable. To do so, we generate and execute new angular velocity control trajectories (distinct from those used to construct the training dataset) following the same procedure as in the training data generation. From the resulting motion sequences, we extract the three goal configurations that are maximally distant from the initial configuration.

The controller gains for the \ac{rolnn}-based control law are set via the diagonal entries of the latent gain matrices $\checkKp\!=\!\mathrm{diag}(0.36, 0.036, 0.03, 0.24, 1.8, 0.024)$ and $\checkKd\!=\!\mathrm{diag}(0.036, 0.0036, 0.003, 0.024, 0.18, 0.0024)$.

Notice that we cannot compare the \ac{rolnn}-based controller with the same baselines as for the augmented pendulum, as (1) the analytical \ac{fom} of the monkey puppet is fully unknown, and (2) a model-free \ac{pd} controller in this underactuated-actuated scenario would require an analytical actuation matrix $\bm B(\bm q)$, which is also unknown for this system.

\subsection{Additional Results}
\label{appendix:results}
This section presents additional experimental results on both the augmented pendulum and the underactuated plush monkey puppet from Sec.~\ref{sec:experiments}.

\subsubsection{Simulated Augmented Pendulum from Sec.~\ref{sec:experiments:pendulum}}\hfill\\
\textbf{Regulation.} 
Fig.~\ref{fig:bd_pend:regulation_pend_dofs} shows the evolution of the $3$ pendulum \acp{dof} for $3$ different reference configurations. As in Sec.~\ref{sec:experiments:pendulum}, we compare the performance of the \ac{rolnn}-based control law with a PD controller on the full system and a linear-intrusive controller. Note that the median error of $10$ of these trajectories was reported in Fig.~\ref{fig:bd_pend:regulation_errors_median_over_time}. Fig.~\ref{fig:bd_pend:regulation_pend_dofs} shows that all controllers stabilize the system at the reference configurations. Importantly, the two model-based controllers feature increased precision compared to the \ac{pd}, which only regulates up to steady state error, see, e.g., the middle and last $q_2$ trajectories. As expected, the linear-intrusive controller converges faster than the \ac{rolnn}-based control law as it reduces the known dynamics of the system to the known pendulum \acp{dof}. The small oscillations in the \ac{rolnn}-based controller arise due to the influence of the mesh \acp{dof} in the learned latent model.
\begin{figure}
    \centering
    \adjustbox{trim=0.6cm 0.3cm 0.5cm 0.2cm}{
            \includesvg[width=.95\linewidth]{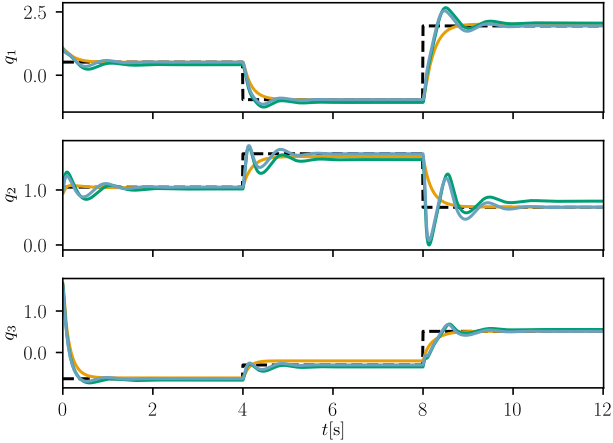}  
        }
    \caption{Regulation trajectories obtained for the $3$ pendulum \acp{dof} for the \ac{rolnn} (\solidblueoneline), linear-intrusive (\solidorangetwoline), and model-free \ac{pd} (\solidgreenthreeline) controllers over $3$ consecutive reference configurations (\dashedblackline).
    }
    \label{fig:bd_pend:regulation_pend_dofs}
\end{figure}

\textbf{Tracking.}  Fig.~\ref{fig:bd_pend:tracking_errors_median_over_time} accompanies Fig.~\ref{fig:bd_pend:tracking_errors_latent_median_over_time} from the main text by displaying the median latent tracking errors $\|\checkbme\|$ over time over the $10$ testing trajectories for the \ac{rolnn} and hybrid \ac{rolnn} controllers. We observe that both controllers achieve small and approximately constant tracking errors. We hypothesize that the hybrid \ac{rolnn} controller leads to a slightly lower latent error as the full-order \ac{pd} term compensates for small nullspace errors, thus reducing cross-couplings in the latent space.
\begin{figure}
    \centering
    \adjustbox{trim=0.5cm 0.3cm 0.5cm 0.2cm}{
            \includesvg[width=.95\linewidth]{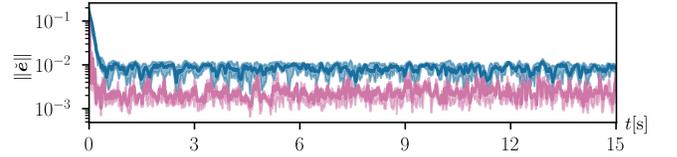}  
        }
    \caption{Latent tracking error (median and quartile) of the \ac{rolnn} (\solidblueoneline) and hybrid \ac{rolnn} (\solidpinkfourline) controllers over $10$ reference trajectories.
    }
    \label{fig:bd_pend:tracking_errors_latent_median_over_time}
    \vspace{-0.2cm}
\end{figure}

Finally, we complement Figs.~\ref{fig:bd_pend:ns_dof_tracking}-left, ~\ref{fig:bd_pend:tracking_errors_median_over_time}, and~\ref{fig:bd_pend:tracking_errors_latent_median_over_time} with Fig.~\ref{fig:bd_pend:tracking_errors_median_over_time_on_man} that shows the average errors $\|\bm e\|$ and $\|\checkbme\|$ obtained by the \ac{rolnn}-based controller in the case where the tracking references are first projected onto the embedded submanifold $\varphi(\mathcal T \mathcal Q)$. We observe that the errors are similar when tracking the original and projected references, showcasing that nullspace effects arise both on and off the submanifold. 
\begin{figure}
    \centering
    \adjustbox{trim=0cm 0.2cm 0cm 0.0cm}{
            \includesvg[width=.95\linewidth]{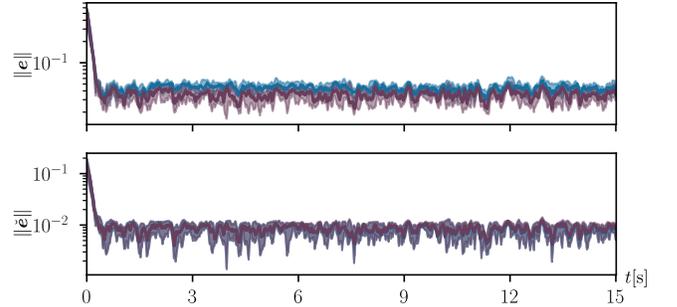}  
        }
    \caption{Tracking error and latent tracking error (median and quartiles) when tracking the original references (\solidblueoneline) and tracking the projected reference $\in \mathcal T \mathcal T \mathcal Q$ (\solidaubergineline) with \ac{rolnn}-based controller.
    }
    \label{fig:bd_pend:tracking_errors_median_over_time_on_man}
\end{figure}

\subsubsection{Underactuated Puppet from Sec.~\ref{sec:experiments:monkey}}
Fig.~\ref{fig:monkey:frames_regulation_3monkeys_version} shows the resulting trajectories of the \ac{rolnn}-based control law on the regulation task, featuring an additional third trajectory compared to Fig~\ref{fig:monkey:frames_regulation} in the main text.
\begin{figure}[]
    \centering
    \includegraphics[width=0.9\linewidth]{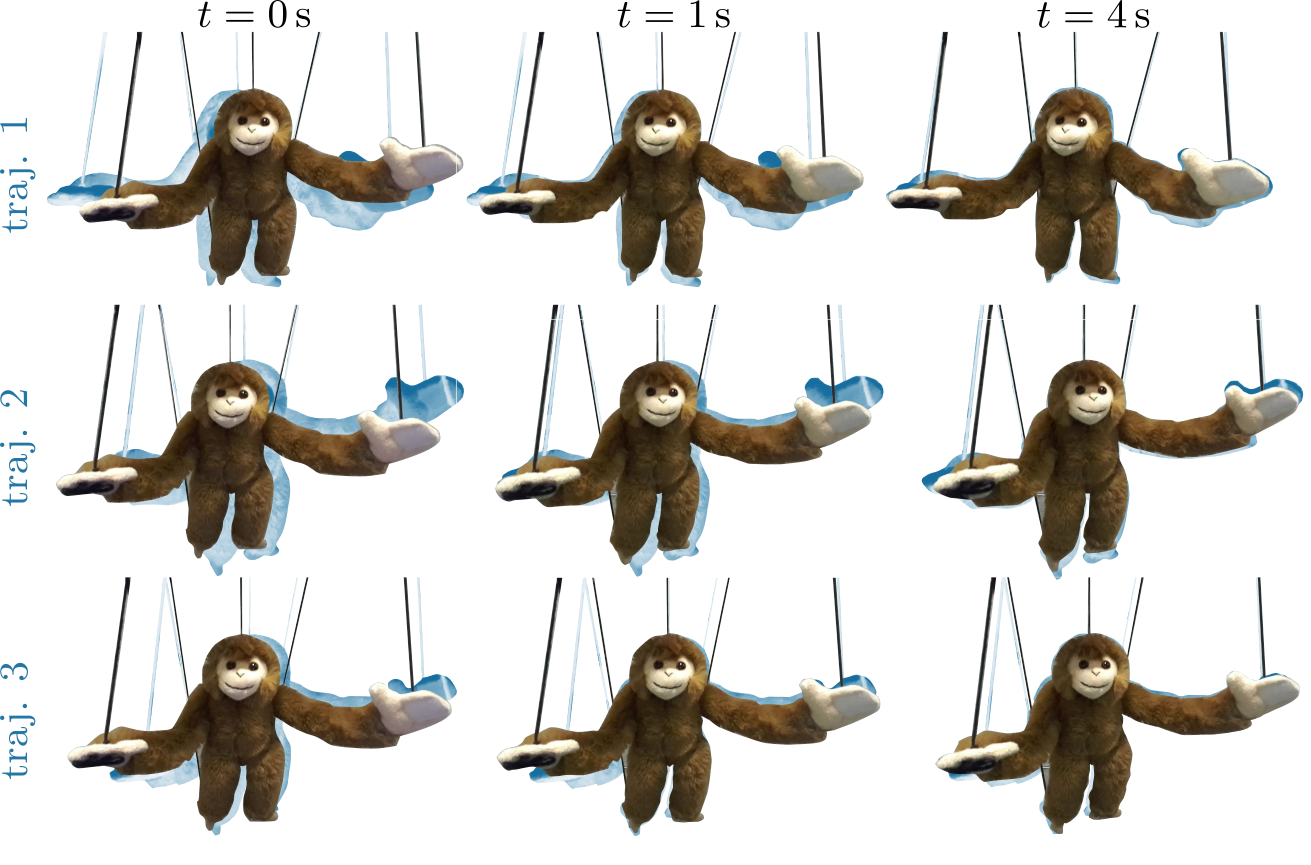}
    \vspace{-0.4cm}
    \caption{Regulated trajectories of a puppet generated by the \ac{rolnn}-based controller for three different reference configurations. 
    }
    \vspace{-0.2cm}
    \label{fig:monkey:frames_regulation_3monkeys_version}
\end{figure}
Fig.~\ref{fig:monkey:control_signal_trajs} shows the the control inputs $\bm{u}$, i.e., desired angular velocity vector of the end-effector, computed by the \ac{rolnn}-based control law for each of the three monkey trajectories of Fig.~\ref{fig:monkey:frames_regulation_3monkeys_version}.
We observe that the control input converge to zero when the target reference is reached.

\begin{figure}
    \centering
    \adjustbox{trim=0.6cm 0.3cm 0.5cm 0.0cm}{
            \includesvg[width=.95\linewidth]{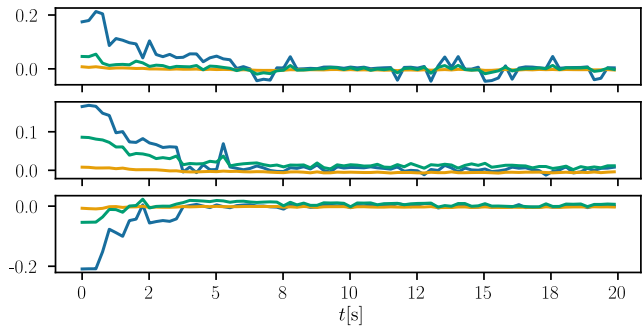}  
        }
    \caption{Commanded angular velocities $r_x$(\solidblueoneline), $r_y$(\solidorangetwoline), $r_z$(\solidgreenthreeline) by the latent \ac{rolnn} based controller on, top to bottom, trajectories 1,2, and 3 from the regulation tasks in Fig.~\ref{fig:monkey:frames_regulation_3monkeys_version}.
    }
    \label{fig:monkey:control_signal_trajs}
\end{figure}

\end{document}